\documentclass{article}
\pdfoutput=1
\PassOptionsToPackage{numbers, compress}{natbib}

\usepackage[preprint]{neurips_2022}

\usepackage[utf8]{inputenc} 
\usepackage[T1]{fontenc}    
\usepackage{hyperref}       
\usepackage{url}            
\usepackage{booktabs}       
\usepackage{amsfonts}       
\usepackage{nicefrac}       
\usepackage{microtype}      
\usepackage{xcolor}         

\usepackage{amsmath}
\usepackage{amsthm}
\usepackage{mathtools}
\usepackage{subfigure}
\usepackage{amsfonts}
\usepackage{autobreak}
\usepackage{makecell}
\usepackage{array,multirow,graphicx, adjustbox}
\usepackage{booktabs}
\usepackage{dcolumn}
\DeclareMathOperator*{\argmax}{arg\,max}

\theoremstyle{plain}
\newtheorem{theorem}{Theorem}[section]
\newtheorem{proposition}[theorem]{Proposition}

\theoremstyle{definition}
\theoremstyle{remark}
\newtheorem{remark}[theorem]{Remark}

\definecolor{mine}{RGB}{98,98,98} %
\definecolor{td3}{RGB}{55,90,113} 
\definecolor{bear}{RGB}{221,191,93} 
\newcolumntype{d}[1]{D{.}{.}{#1}}

\usepackage{listings}
\usepackage{courier}
\usepackage{wrapfig}

\definecolor{numbercolor}{rgb}{0.5,0.5,0.5}
\definecolor{codebackground}{RGB}{251, 251, 251}
\definecolor{commentscolor}{RGB}{100,100,100}

\lstdefinestyle{mystyle}{
  backgroundcolor=\color{codebackground}, commentstyle=\color{commentscolor},
  keywordstyle=\color{blue},
  numberstyle=\tiny\color{numbercolor},
  basicstyle=\fontsize{7}{10}\ttfamily,
  breakatwhitespace=false,
  breaklines=true,                 
  captionpos=b,                    
  keepspaces=true,                 
  numbers=left,                    
  numbersep=6pt,                  
  showspaces=false,                
  showstringspaces=false,
  showtabs=false,                  
  tabsize=4,
  otherkeywords={norm, .mse_loss, .relu, .actor, clone, detach, repeat, requires_grad_, .rand, .size, .mean, True, .to, .critic, .grad}
}

\title{
Robust Offline Reinforcement Learning\\ with Gradient Penalty and Constraint Relaxation
}

\author{
Chengqian Gao$^1$, Ke Xu$^{2}$, Liu Liu$^{2}$, Deheng Ye$^{2}$, Peilin Zhao$^{2}$, Zhiqiang Xu$^{1}$,  \\
$^{1}$ MBZUAI \\ 
$^{2}$ Tencent AI Lab\\
\texttt{chengqian.gao@mbzuai.ac.ae} \\ 
\texttt{\{kaylakxu, leonliuliu, dericye, masonzhao\}@tencent.com}\\ \texttt{zhiqiangxu2001@gmail.com}\\
}

\begin{document}
\maketitle

\begin{abstract}
A promising paradigm for offline reinforcement learning (RL) is to constrain the learned policy to stay close to the dataset behaviors, known as policy constraint offline RL. 
However, existing works heavily rely on the purity of the data, exhibiting performance degradation or even catastrophic failure when learning from \textit{contaminated datasets} containing impure trajectories of diverse levels. e.g., expert level, medium level, etc., while offline contaminated data logs exist commonly in the real world. 
To mitigate this, we first introduce \textit{gradient penalty} over the learned value function to tackle the exploding Q-functions. We then relax the closeness constraints towards non-optimal actions with \textit{critic weighted constraint relaxation}. 
Experimental results show that the proposed techniques effectively tame the non-optimal trajectories for policy constraint offline RL methods, evaluated on a set of contaminated D4RL Mujoco and Adroit datasets. 
\end{abstract}

\section{Introduction}



Effective offline reinforcement learning (RL) should be able to extract policies with the maximum possible utility out of the static demonstrations without interacting with the environment \cite{lange2012batch, fujimoto2019off, levine2020offline}. One typical way of offline RL is to use policy constraint, enforcing the learned policy to stay close to the behavior policy that generated the dataset, involving various closeness metrics \cite{fujimoto2019off, kumar2019stabilizing, wu2019behavior, kostrikov2021offline}. 

However, we find many policy constrained offline RL methods suffer \textit{performance degradation} and even \textit{catastrophic failure} (please see Figure~\ref{fig:catastrophic_failue}) when trained on datasets containing different levels of policy trajectories. 
For example, methods in Figure \ref{fig:performance_degradation} show better performance on the expert dataset while achieving lower scores on the expert-medium dataset. This is undesired as the medium-expert datasets contain more dynamics, i.e., both expert and medium-level data \cite{fu2020d4rl}. 

\begin{wrapfigure}{r}{0.46\textwidth}
    \centering
    \vspace{-0.5cm}
    \includegraphics[width=\linewidth]{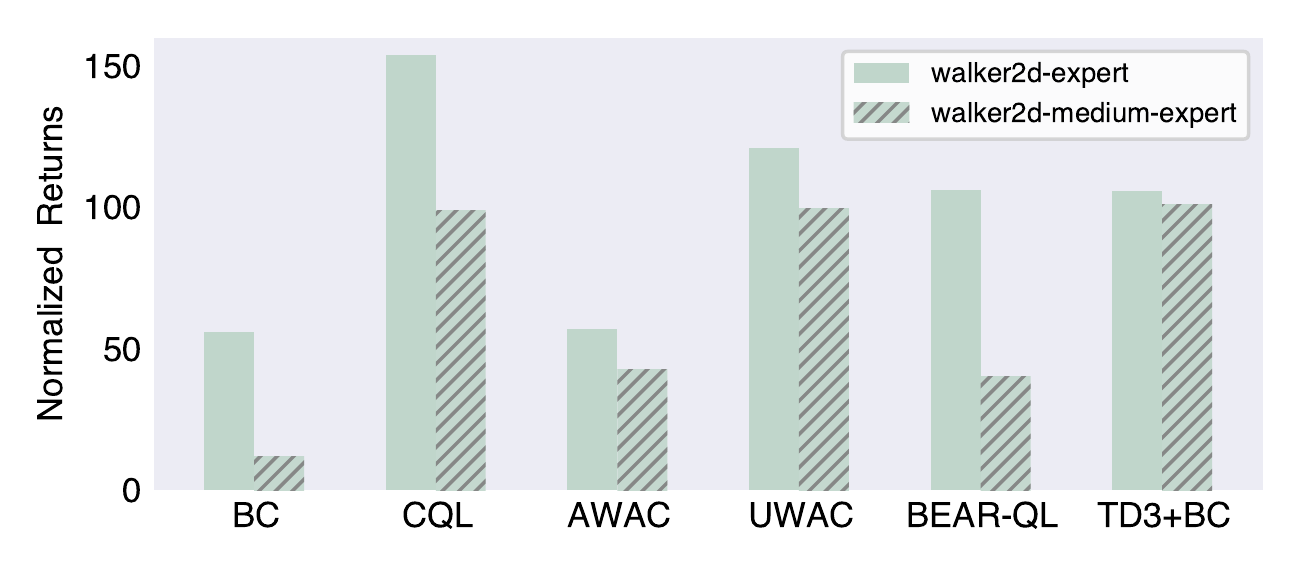}
    \vspace{-0.9cm}
    \caption{Performance degradation. We report the paper results of CQL~\cite{kumar2020conservative}, AWAC~\cite{nair2020accelerating}, UWAC~\cite{wu2021uncertainty}, TD3+BC~\cite{fujimoto2021minimalist}. For BEAR-QL~\cite{kumar2019stabilizing}, we report the result from D4RL~\cite{fu2020d4rl}.\vspace{-0.5cm}}
    \label{fig:performance_degradation}
\end{wrapfigure} 

In fact, many real-world applications demand robust offline RL algorithms, such as robotic controlling tasks with datasets for multiple tasks or incomplete demonstrations \cite{sun2019adversarial, fu2020d4rl}, recommendation tasks with datasets containing non-user logs~\cite{gunes2014shilling, huang2021data}, and autonomous driving tasks with trajectories with various levels. In these cases, the dataset contains expert demonstrations and trajectories from non-experts who have not mastered the task~\cite{tangkaratt2018improving}. Filtering out non-expert trajectories with human effort is either expensive or impossible, necessitating robust offline RL algorithms.


\begin{figure}[ht]
    \centering
    \includegraphics[width=\linewidth]{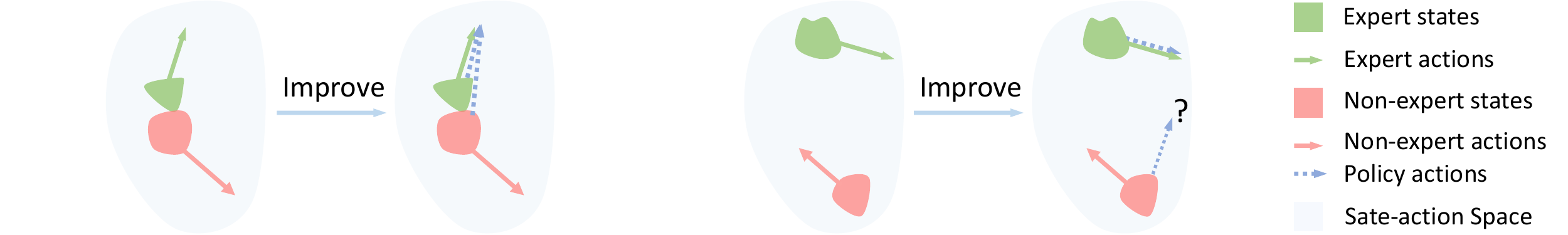}
    \vspace{-0.4cm}
    \caption{Non-expert demonstrations inhibit in two different ways. For non-expert states overlap with expert dataset states(left), they influence the policy improvement in a supervised fashion. For non-expert states far aways from expert dataset states, policy improvement may lead to OOD actions.}
    \vspace{-0.5cm}
    \label{fig:state_overlap}
\end{figure}
Why do the observed performance degradation and catastrophic failure occur? 
To answer this question, we first introduce \textit{contaminated datasets}, which contain trajectories from both expert and non-expert behavior policies, including \textit{expert-medium}, \textit{expert-cloned}, and \textit{expert-random} datasets. 
By analyzing the learning behaviors on such datasets, we identify two key paths by which non-expert trajectories inhibit policy constraint based offline RL.

First, non-expert data can inhibit policy constraint based offline RL in a supervised manner (see Figure \ref{fig:state_overlap}, left). 
The closeness constraint explicitly regresses the policy to both expert and non-expert actions, leading to a compromised policy when the states visited by expert and non-expert policies show significant overlaps \cite{wu2019imitation}. 
To tackle this issue, we propose critic weighted constraint relaxation (\textbf{+ CR}), which leverages a polished Q-function to relax the harmful closeness constraint towards non-expert actions.

A more important finding of this work is that non-expert trajectories may destroy the learned Q-function via out-of-distribution (OOD) actions (Figure \ref{fig:state_overlap}, right). 
Policy improvements on the contaminated dataset make the learned policy closer to dataset expert actions while moving it away from non-expert decisions. 
This implicitly leads to the failed closeness constraint on non-expert states when expert and non-expert states follow different distributions. 
Such failed closeness constraints may result in OOD actions and in turn give rise to unstable Q-values (Theorem \ref{theo:constraintMDP}), sharp Q-function gradients, and finally catastrophic failures (Figure \ref{fig:catastrophic_failue}). We thus introduce the gradient penalty technique (\textbf{+ GP}) to suppress the observed sharp Q-function gradients induced from the failed closeness constraint. To justify the proposed GP technique, we show that there should be an upper bound for the norm of (optimal) Q-function gradients (Theorem \ref{theo:lipschitzness}). 

We integrate the proposed two techniques on the top of BEAR-QL \cite{kumar2019stabilizing} and TD3+BC \cite{fujimoto2021minimalist}, attaining BEAR++ and TD3BC++. 
Evaluations on the contaminated datasets for D4RL mujoco and adroit tasks demonstrate that the proposed techniques together could serve as a general plugin to tame the policy constrained offline RL algorithms. 

\section{Preliminaries} 
\paragraph{RL.}
A Markov decision process (MDP) can be represented by $M=\langle \mathcal{S}, \mathcal{A}, T, d_0, r, \gamma \rangle$, with state space $\mathcal{S}$, action space $\mathcal{A}$, transition probability $T(s_{t+1} | s_t, a_t)$, initial state distribution $d_0$, reward function $r(s_t, a_t)$, and discount factor $\gamma$. 
RL methods aim to find a policy $\pi(a_t | s_t)$, to maximize the expected (discounted) cumulative reward $ \mathbb{E}_{\tau \sim p_{\pi}(\tau)} \Big[ \sum_{t=0}^{|\tau|} \gamma^t r(s_t, a_t)\Big]$, with the trajectory distribution $p_{\pi}(\tau) = d_0(s_0) \prod_{t=0}^{|\tau|} \pi(a_t|s_t) T(s_{t+1}| s_t, a_t)$. 


\paragraph{Offline RL.}
Offline RL algorithms aim to obtain policies from static set of interactions, $D=\{(s_t, a_t, s_{t+1}, r(s_t,a_t))\}_{t=0}^N$. 
One main challenge in offline RL is the distribution shift issues \cite{fu2019diagnosing} or the extrapolation error \cite{fujimoto2019off} during training. 
For example, a Q-function is trained on \textit{dataset actions} $ \mu(a_t | s_t)$ but is evaluated on \textit{policy actions} $\pi(a_{t+1} | s_{t+1})$:
\begin{equation}
    \label{policy_evaluation}
    Q^{k+1}(s_t, a_t) = \operatornamewithlimits{\mathbb{E}}_{s_t,a_t,r,s_{t+1} \sim \mathcal{D}} [r + \gamma Q^{k}(s_{t+1}, \pi(a|s_{t+1}))]
\end{equation}
The learned policy may generate \textit{out-of-distribution} (OOD) actions that differ from the dataset action since its optimization objective makes no other guarantee except generating high-value actions:
\begin{equation}
    \label{policy_improvement}
    \pi^{k+1} = \argmax_{\pi} Q(s_t, \pi(s_t))
\end{equation}
Policy improvements implicitly drive the policy to explore OOD actions \cite{hu2021actor}, and policy evaluation exploits these OOD actions and in turn affects the policy improvement.

\paragraph{Policy constrained offline RL.}
One avenue towards offline RL is enforcing the learned policy to stay close to the behavior policy that generated the dataset:
\begin{equation}
    \pi^{k+1} = \argmax_{\pi} Q(s_t, \pi(s_t)), \qquad s.t. \operatorname{closeness \ constraint}
\end{equation}
Without loss of generality, we study and try to address the performance degradation and catastrophic failure issues in two policy constraint based offline RL algorithms, TD3+BC \cite{fujimoto2021minimalist} and BEAR-QL \cite{kumar2019stabilizing}. 

TD3+BC adds a behavior cloning term on the top of TD3 \cite{fujimoto2018addressing}, resulting in:
\begin{equation}
    \begin{aligned}
    \label{eq:td3_bc}
    \pi = \argmax_{\pi} \mathbb{E}_{(s_t,a_t)\sim \mathcal{D}} \Big[ 
    \frac{\alpha}{\mathbb{E}[|Q(s_t, a_t)|]} Q(s_t, \pi(s_t)) - (\pi(s_t) - a_t)^2\Big],
\end{aligned}
\end{equation}
where $\alpha$ is a hyperparameter controlling the strength of the regularizer.

BEAR-QL constrains the learned policy to have non-negligible support under the data distribution:
\begin{equation}
\begin{aligned}
    \pi = \argmax_{\pi} \mathbb{E}_{s_t \sim \mathcal{D}} \Big[
    Q(s_t, \pi(s_t)) \Big] \operatorname{ s.t. } \ \mathbb{E}_{s_t\sim \mathcal{D}} [\operatorname{MMD} (\beta(s_t), \pi(s_t))] \leq \epsilon, 
\end{aligned}
\end{equation}
with $\beta$ for approximating the behavior policy and $\epsilon=0.05$ for a threshold.

\section{Catastrophic failure happens with exploding Q-gradients}
Policy constrained offline RL methods fail to learn meaningful policies on contaminated datasets that contain significantly multi-modal state distributions, e.g., expert-cloned, expert-random. We call this catastrophic failure as it happens with a destroyed Q-function with exploding gradients.

\begin{wrapfigure}{r}{0.54\textwidth}
    \centering
    \vspace{-0.5cm}
    \includegraphics[width=\linewidth]{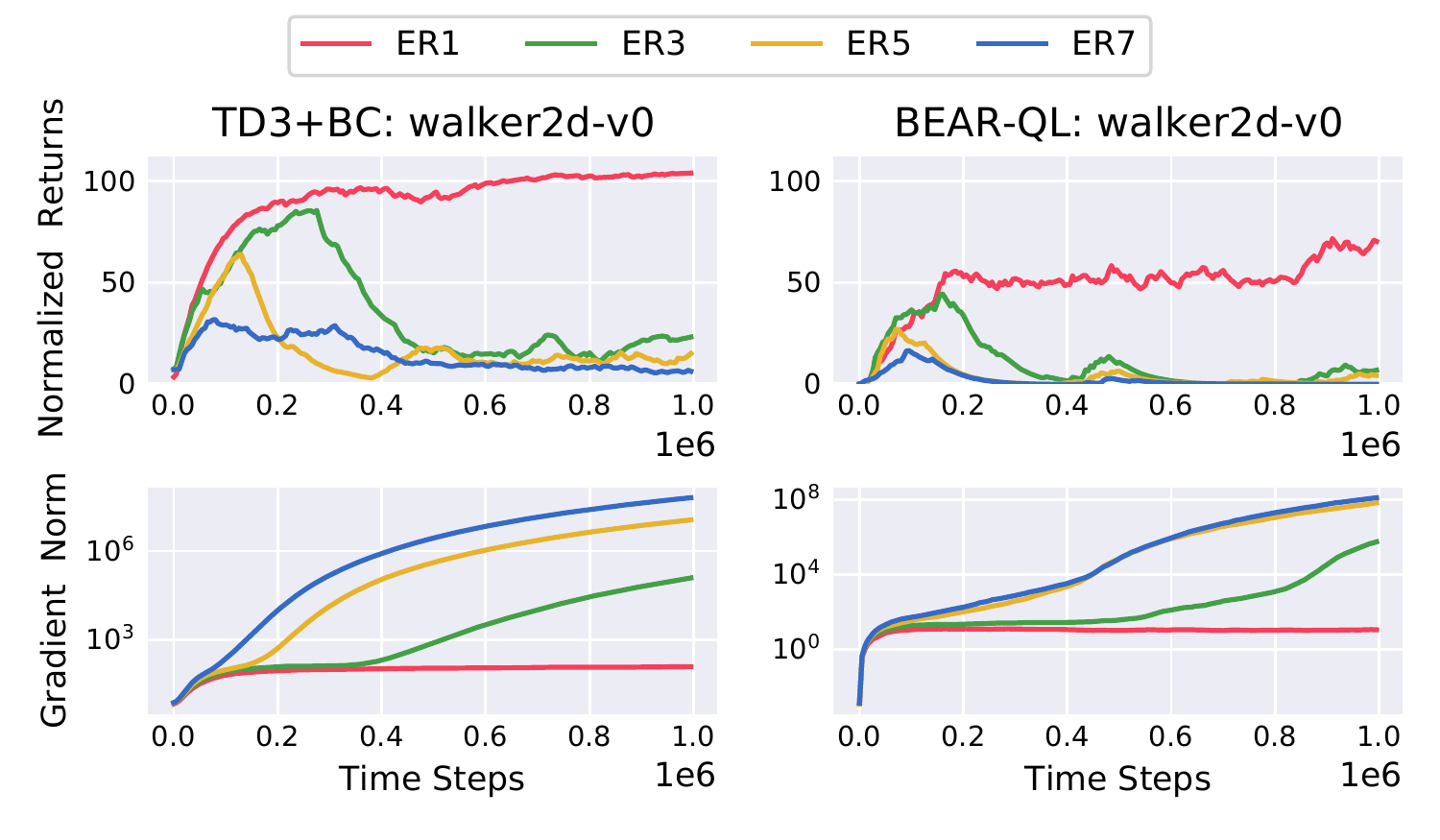}
    \vspace{-0.5cm}
    \caption{Catastrophic failure. \textbf{ER1} is short for expert-random-10, which denotes a contaminated dataset with 10\% random trajectories and 90\% expert demonstrations.}
    \label{fig:catastrophic_failue}
    \vspace{-0.5cm}
\end{wrapfigure}

\subsection{Catastrophic failures on contaminated datasets}
In order to mimic real-world logs that contain multi-level trajectories, we introduce contaminated datasets, which can be generated by contaminating an expert dataset with non-expert demonstrations. For instance, ER3, short for expert-random-30, refers to a dataset in which 70 percent are expert trajectories and 30 percent are from random behavior policies. Please refer to Appendix \ref{app:the_modified_dataset} for the detailed statistics and discussion about the contaminated datasets. 

We run TD3+BC and BEAR-QL on a contaminated dataset, walker2d-expert-random-v0, as depicted in Figure \ref{fig:catastrophic_failue}. The catastrophic failure occurs after the percentage of random data exceeds 30\%.

\subsection{Analysis with distribution-constrained Q-iteration}
Why do catastrophic failures occur, and always after the learned policy's performance has improved?
In order to give some insights, we use the analysis tools from~\cite{kumar2019stabilizing}, which involves a constrained Bellman backup operator, defined as:
\begin{equation}
     \mathcal{T}^\Pi Q(s_t,a_t) := \mathbb{E} \big[ r + \gamma \max_{\pi \in \Pi} \mathbb{E}_{T(s_{t+1} | s_t, a_t)} [V_{\pi}(s_{t+1})] \big], 
\end{equation}
with state value function $V_{\pi}(s_t) := \mathbb{E}_{\pi} [Q(s_t, \pi(a_t|s_t))]$ and a $\Pi$ to restrict the set of policies.


\vspace{0.1cm}
\begin{theorem}
\label{theo:constraintMDP}
The performance of distribution-constrained Q-iteration can be bounded as:
\begin{equation}
    \lim_{k \to \infty} \mathbb{E}_{d_0}\big[\Big|V^{\pi^k}(s_t) - V^\Pi(s_t)\Big|\big] \le 
    \frac{2\gamma}{(1-\gamma)^2} C_{\Pi, \mu} \mathbb{E_{\mu}}\Big[ \max_{\pi \in \Pi} \mathbb{E}_\pi [\delta(s_t,a_t)]\Big]
    \label{inequality}
\end{equation} 
with the concentrability coefficient $C_{\Pi, \mu}$ for quantifying how far the conditional distribution of the policy action $\pi(a_t |s_t) \sim \Pi$ is from the corresponding dataset action $\mu(a_t |s_t)$. $V^{\Pi}$ denotes the fixed point of $\mathcal{T}^\Pi$, $d_0$ denotes the initial state distribution.
\end{theorem}

To understand why the catastrophic failure happens, we simplify the concentrability coefficient $C_{\Pi, \mu}$~\cite{munos2008finite} as the distance between the decisions from $\pi$ and $\mu$. 

Before the discussion, we first consider learning from a pure dataset generated by policies with similar decision-making capabilities. The policy improvement implicitly drives the learned policy out of the dataset distribution~\cite{hu2021actor}, resulting in a large concentrability coefficient~\cite{kumar2019stabilizing} for all dataset states. 
Policy constraint based offline RL algorithms force the learned policy $\pi \sim \Pi$ close to the behavior policy $\mu$, yielding a low concentrability coefficient and thus making it possible to learn RL policy from static datasets. 

For the contaminated dataset with two (or more) behavior policies with significantly different decision-making capacities, policy constrained offline RL faces a dilemma: 

\textit{Adhering to the non-expert dataset decisions (low $C_{\Pi, \mu}$) leads to bad policies, but driving out of the non-expert trajectories (large $C_{\Pi, \mu}$) faces OOD actions.}

\begin{wrapfigure}{r}{0.4\textwidth}
    \centering
    \vspace{-0.5cm}
    \includegraphics[width=\linewidth]{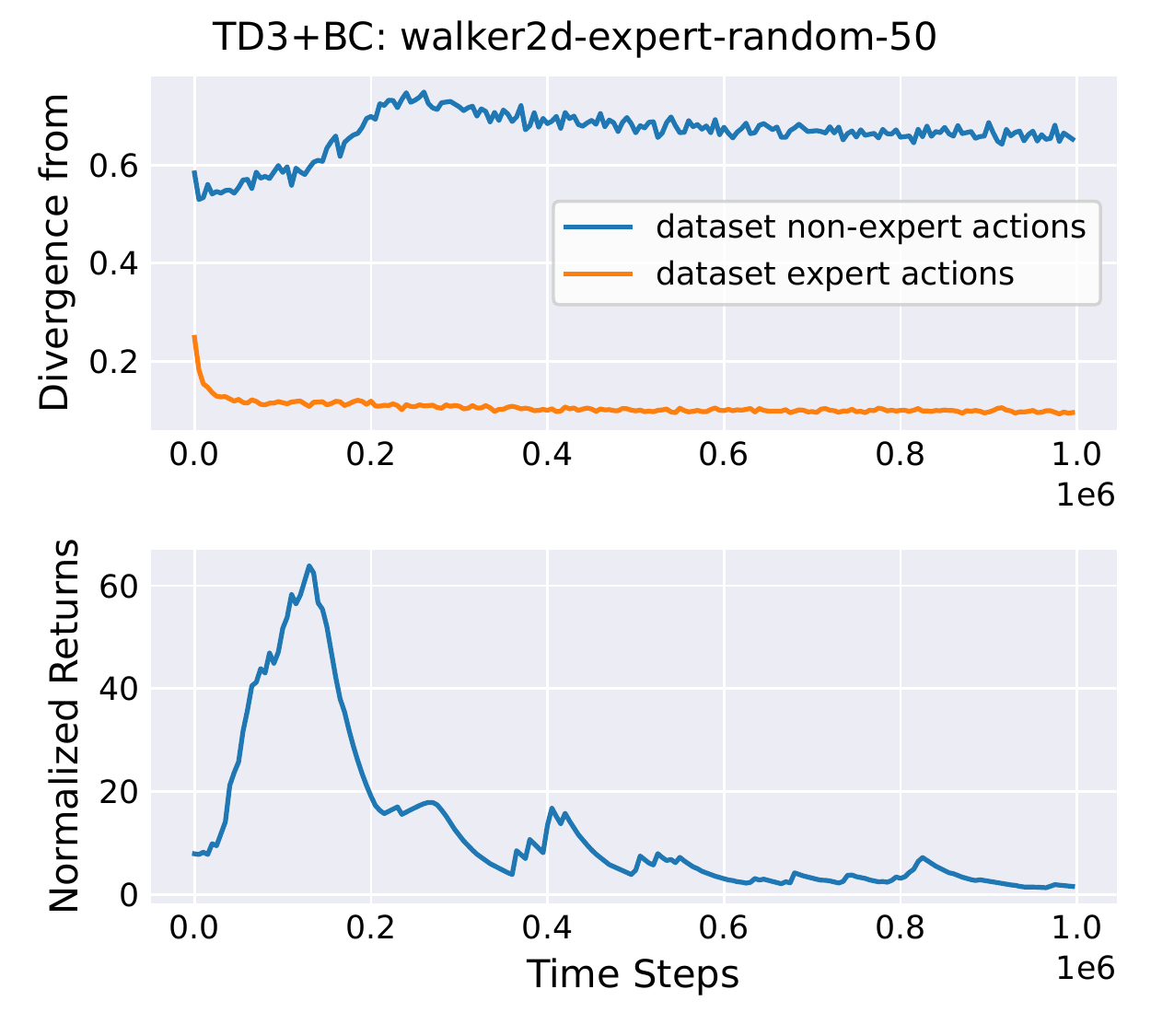}
    \vspace{-0.5cm}
    \caption{The failed closeness constraint on non-expert state-action pairs is correlated with catastrophic failures. \textbf{Divergence:} the 75th percentile of the squared error between the policy decisions and the corresponding dataset actions.}
    \label{fig:the_filed_closeness_constraint}
    \vspace{-0.5cm}
\end{wrapfigure}

In this case, behavior policies show different state visitions, as depicted by Figure~\ref{fig:state_overlap} and~\ref{fig:state_overlaps}. Policy improvements may implicitly drive the learned policy to be different from the non-expert dataset actions and thus leads to the failed closeness constraint on non-expert states, resulting in OOD actions and large $C_{\Pi, \mu}$. We visualize this process in Figure~\ref{fig:the_filed_closeness_constraint}. The catastrophic failure is correlated with the failed constraint towards non-expert decisions, i.e., the increasing divergence between the policy and non-expert dataset actions. Please note that throughout the learned policy stays close to dataset expert actions. 


The closeness constraint on non-expert state-action pairs is destroyed by the policy improvement, which is why it always occurs after achieving good performance (please see Figure \ref{fig:catastrophic_failue}).
The failed closeness constraint leads to OOD actions, which are recognized as the main challenge for offline RL that induces erroneous Q-values, overestimation problems, and bad policies. However, due to the dilemma mentioned above, finding a proper closeness metric for contaminated datasets is hard. 

How to save the policy when OOD actions are inevitable? To the best of our knowledge, this work is the first to observe that OOD actions are correlated with extreme sharpness of the Q-values (with respect to actions) and not just overly large values. This inspires us to alleviate the impact of OOD actions from the perspective of gradient regularity.

\section{Recovering from catastrophic failure via gradient penalty}
In this part, we introduce a gradient penalty to minimize the impact of OOD actions induced by the policy improvements on non-expert state-action pairs. Then we give proof to support the proposed gradient penalty. Finally, we discuss the difference between our method and a previous work. 

\subsection{Penalizing the unstable gradients}
Recall that the policy improvement step with neural network approximation is:
\begin{equation}
    \pi^{k+1}_{\theta} = \argmax_{\pi} Q(s_t, \pi(s_t))  
\end{equation}
In practice we run gradient ascent over the parameter space:
\begin{equation}
    \label{equation:gradient_affects}
    \theta = \theta + \alpha \cdot \nabla_{a_t} Q(s_t,a_t) \Big |_{a_t = \pi_{\theta}(s_t)} \cdot \nabla_{\theta} \pi_\theta(s_t)
\end{equation}

\begin{figure*}[t]
    \centering
    \includegraphics[width=\linewidth]{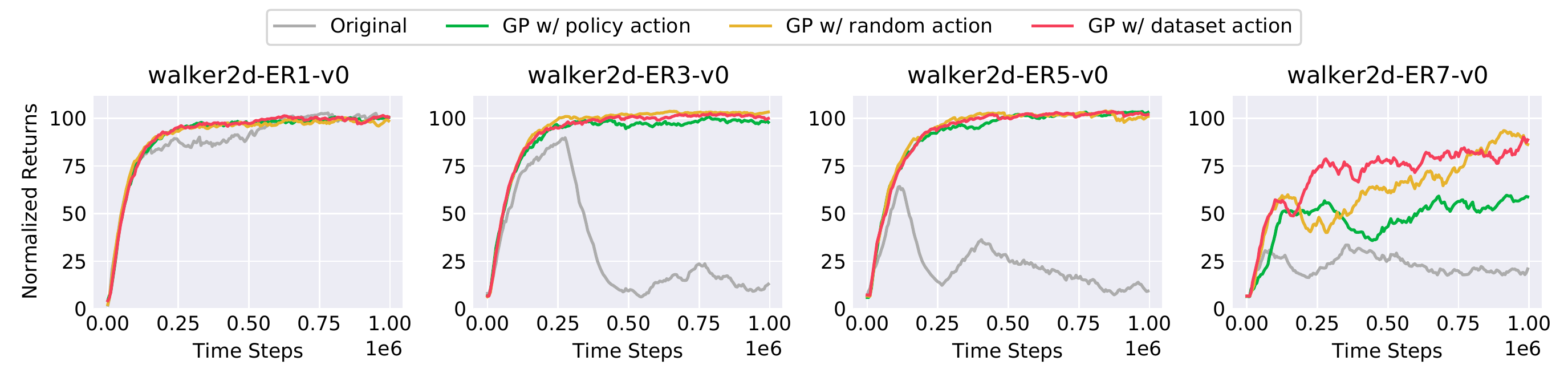}\vspace{-0.2cm}
    \caption{Performance of original TD3+BC algorithm (\textbf{Original}), and TD3+BC with gradient penalty w.r.t. actions from different sampling strategies (\textbf{Dataset action} for $a \sim \mu$, \textbf{Policy action} for $a \sim \pi$, and \textbf{Random action} for $a \sim \mathcal{A}$) on the contaminated D4RL datasets. \textbf{ER1} is short for expert-random-10. Best viewed in color. }
    \label{fig:samplingstragety_Qweights} \vspace{-0.5cm}
\end{figure*} 

We then recall the reasons for catastrophic failure. The improved policy generates OOD actions on non-expert states (with large $C_{\Pi, \mu}$), loosing the performance bound \ref{inequality} and resulting in unstable Q-values. The policy that derived from the misleadingly sharp Q-function gradients in turn produces unseen actions. 
To break the pathological loop, we propose our first modification for policy constrained offline RL methods, i.e., gradient penalty term in the critic loss:
\begin{equation}
\label{euqation:gradient_penalty}
    \mathcal{L}_{GP} = \lambda_{GP} \mathop{\mathbb{E}}\limits_{s_t\sim \mathcal{D}, \ a} \big[ \operatorname{ReLU}\big( \big\| \nabla_{a} Q(s_t,a) \big\|_F  - 1 \big) \big]^2
\end{equation}

We introduce a one-sided penalty to encourage the norm of the Q-function gradient w.r.t. non-expert action stays below $1$ while avoiding over-punishment for expert alike actions. $\lambda_{GP}$ controls the contribution of the gradient penalty term. In order to improve computational efficiency, we perform a gradient penalty in every $N$ training steps (we empirically set $N$ to 5 in our experiments).

Note that we do not specify the sampling distribution for action $a$, as we find there is no significant performance difference between the following three sampling strategies: 1) the current policy action $a \sim \pi$ , 2) the dataset action distribution $a \sim \mu$, and 3) random sampling over the action space $a \sim \mathcal{A}$. We will discuss the different motivations behind these choices later. 

\subsection{Lipschitz property of the learned Q-function}
To justify the proposed gradient penalty technique, we here prove that the Frobenius norm of the learned Q-function gradients w.r.t. input actions should be bounded.

\begin{theorem}
\label{theo:lipschitzness}
Suppose a policy $\pi(a_t | s_t)$ on an MDP $M= \langle \mathcal{S}, \mathcal{A}, r, \gamma, T \rangle$ satisfies the inequality $ \Big\| \frac{\partial{\pi(a_{t+1} | s_{t+1})}}{\partial{a_t}} \Big\|_F \leq L_{\pi,T} < 1$ and the reward function $r(s_t, a_t)$ satisfies $\Big\| \frac{\partial{r(s_t, a_t )}}{\partial{a_t}}  \Big \| \leq L_r$. If we denote the dimension of the action space as $N$, then the magnitude of the gradient of the learned Q-function w.r.t. action can be upperbounded as:
\begin{equation}
\label{euqation:upperbound}
    \Big\| \nabla_{a_t} Q^{\pi}(s_t,a_t) \Big\|_F \leq \frac{\sqrt{N} L_r}{1- \gamma L_{\pi, T}}
\end{equation}
\end{theorem}
\begin{proof}See Appendix \ref{app:lipschitzness}. \end{proof}
\vspace{-0.3cm}

\remark Theorem \ref{theo:lipschitzness} holds for offline RL setting as the offline MDP is equal to an modified online MDP with a constrained Bellman backup operator~\cite{kumar2019stabilizing}. 
It tells that the Q-prediction should not vary much for a perturbation in input action, suggesting that the observed exploding Q-function gradients is unreasonable and thus motivates the gradient penalty.

\subsection{Difference with Fisher-BRC}
A keen reader may note that proposed gradient penalty looks similar to the Fisher divergence term in Fisher-BRC~\cite{kostrikov2021offline}:
\begin{equation}
    \mathop{\mathbb{E}}\limits_{s_t \sim \mathcal{D}} \Big[  \operatorname{Fisher}\big( \frac{\exp{Q(s_t, \cdot)}}{\sum_a \exp{Q(s_t,a)}} , \mu(\cdot| s_t) \big)  \Big] =  \nonumber 
    \mathop{\mathbb{E}}\limits_{s_t \sim \mathcal{D}, a \sim \pi_{emb}(\cdot | s_t)} \big[ \| \nabla_a Q(s_t,a) - \nabla_{a} \log \mu(a|s_t) \|_F^2  \big].
\end{equation}

However, they are different as 1) Fisher-BRC utilizes gradients to measure the Fisher information distance between the learned policy and the behavior policy $\mu$, while 2) our method serves to minimize the negative impact of OOD actions. 
With different motivations, our method 1) does not require an entropy regularizer for recovering the Boltzmann policy $\pi_{emb}$, and 2) should be insensitive to the action sampling distribution (while Fisher-BRC needs $a \sim \pi_{emb}$). Results in Figure \ref{fig:samplingstragety_Qweights} verify the second conjecture, in which three types of sampling strategies show no performance difference for expert-random-10 (ER1), ER3, and ER5 settings. We perform gradient penalty w.r.t random actions in the experiments section.

\begin{figure*}[t]
    \centering 
    \subfigure{\includegraphics[width=0.49\linewidth]{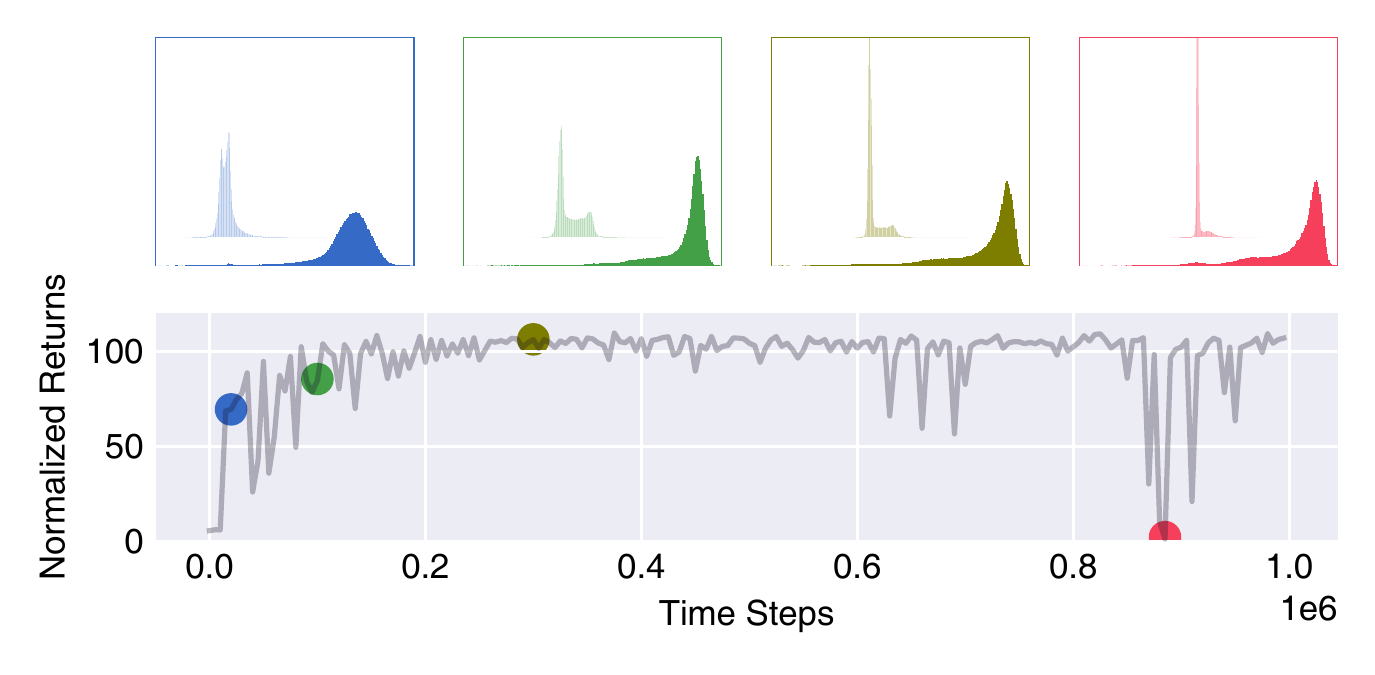}}
    \subfigure{\includegraphics[width=0.49\linewidth]{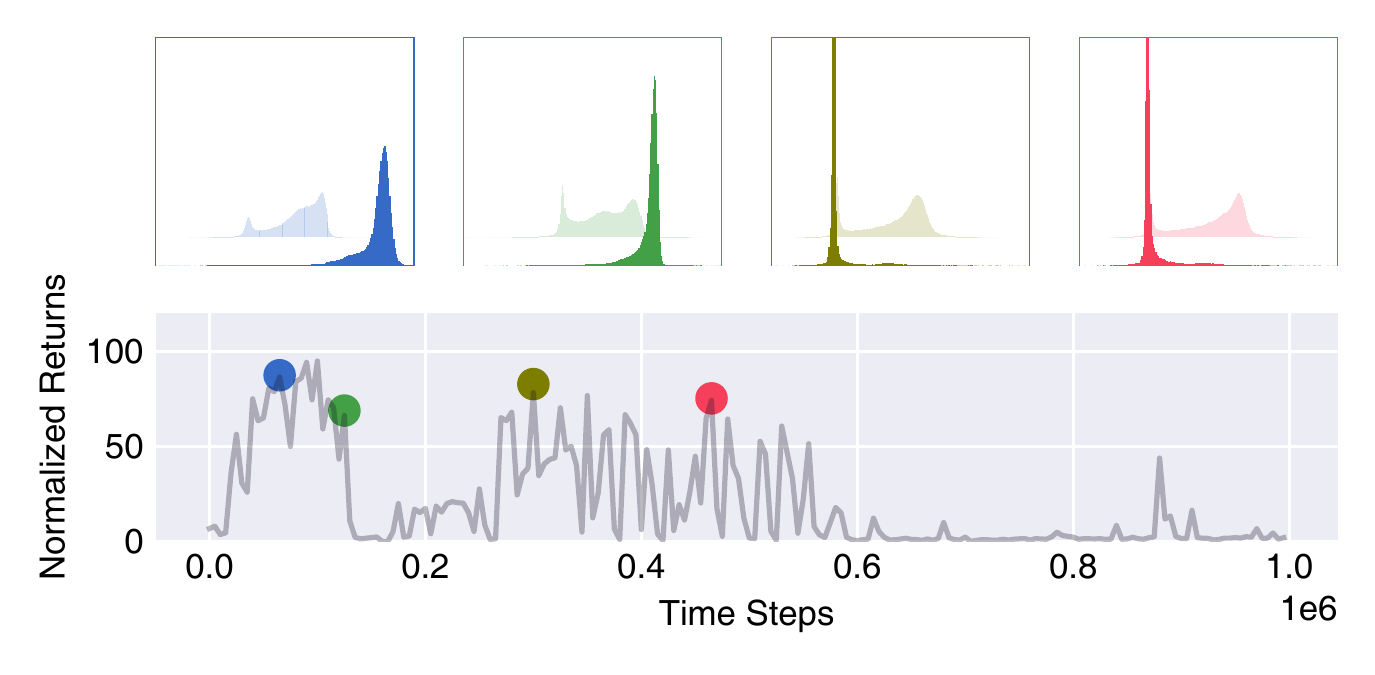}}\vspace{-0.2cm}
    \caption{Q-function with gradient penalty can distinguish expert and random actions. We plot the Q-value distributions of dataset actions (top) in the training process (bottom) of a TD3+BC agent \textbf{with (left)} and \textbf{without (right)} gradient penalty. Diluted and raised histogram for random actions, heavy color for expert actions. Task name is walker2d-expert-random-50-v0. Best viewed in color.\vspace{-0.1cm}}
    \label{fig:Q_value_distributions}
\end{figure*}

\section{Constraint relaxation with the polished Q-function}
The harmful closeness constraints toward non-expert dataset actions prevent the learned policy from optima, see also in ~\cite{wu2019imitation, sasaki2020behavioral}. We further relax the harmful closeness constraint by critic weighted constraint relaxation (+ CR) in this part.

The key challenge to relaxing the harmful constraints is to indicate the optimality of the dataset actions. In offline RL, the learned Q-function might serve this purpose. 
As depicted in Figure \ref{fig:Q_value_distributions}, the polished Q-function could successfully discriminate expert decisions (heavy colors) and random actions (diluted), even when the policy performs not so well (left). On the other hand, without GP, the Q-function is not accurate even if the performance is good (right).

We use the Q-value to indicate the optimality, with a min-max normalization over a mini-batch: 
\begin{equation}
    W(s_t, a_t) = \frac{Q(s_t,a_t) - Q_{min}}{Q_{max} - Q_{min}}  
\end{equation}
We then could rewrite the regularizer term in BEAR-QL as:
\begin{equation}
    \pi = \argmax_{\pi} \mathbb{E}_{s_t\sim \mathcal{D}} \Big[
    Q(s_t, \pi(s_t)) \Big] \ 
    \operatorname{ s.t. } \mathbb{E}_{s_t \sim \mathcal{D}} [\operatorname{MMD} \big(\beta(s_t), \pi(\cdot|s_t)\big) \cdot W\big(s_t, \beta(s_t)\big)] \leq \epsilon, 
\end{equation}
and for TD3+BC we have:
\begin{equation}
\begin{aligned}
    \pi = \argmax_{\pi} \mathop{\mathbb{E}}\limits_{(s_t,a_t)\sim \mathcal{D}} \Big[ 
     \frac{\alpha}{\mathbb{E}[|Q(s_t, a_t)|]} Q(s_t, \pi(s_t)) - (\pi(s_t) - a_t)^2 \cdot W(s_t,a_t)\Big]
\end{aligned}
\end{equation}
Note that we do not propagate gradient through the relaxation weight, $W(s_t,a_t)$.

\section{Experiments}
We proposed two modifications for policy constraint based offline RL: 1) gradient penalty (+ GP) to alleviate the negative impacts of OOD actions induced from the failed closeness constraint and 2) critic weighted constraint relaxation (+ CR) for the harmful closeness constraint, attaining TD3BC++ and BEAR++, on the top of TD3+BC and BEAR-QL.

\begin{table*}[!htp]
\centering
\vspace{-0.2cm}
\caption{Evaluation on the D4RL Mujoco Gym tasks. ER1 is short for Expert-random-10. We rerun all algorithms. With the proposed two techniques, BEAR++ and TD3BC++ could address the performance degradation and catastrophic failures issues. The highest performing scores are bolded.}\vspace{0.1cm}
\label{tab:performance_on_mujoco}
\begin{adjustbox}{width={\textwidth},totalheight={\textheight},keepaspectratio}%
\renewcommand{\arraystretch}{1.2}
\begin{tabular}{ccr@{\hspace{1pt}}lr@{\hspace{1pt}}lr@{\hspace{1pt}}lr@{\hspace{1pt}}lr@{\hspace{1pt}}lr@{\hspace{1pt}}lr@{\hspace{1pt}}lr@{\hspace{1pt}}lr@{\hspace{1pt}}l | r@{\hspace{1pt}}lr@{\hspace{1pt}}l}

\toprule
\textbf{Task}   & \textbf{Setting}  & \multicolumn{2}{c}{\textbf{BC}} & \multicolumn{2}{c}{\textbf{\%BC}} & \multicolumn{2}{c}{\textbf{CQL}} & \multicolumn{2}{c}{\textbf{BEAR-QL}} &  \multicolumn{2}{c}{\textbf{TD3+BC}} & \multicolumn{2}{c}{\textbf{\%TD3+BC}} & \multicolumn{2}{c}{\textbf{Fisher-BRC}} & \multicolumn{2}{c}{\textbf{UWAC}} & \multicolumn{2}{c}{\textbf{IQL}} & \multicolumn{2}{c}{\textbf{BEAR ++}}   & \multicolumn{2}{c}{\textbf{TD3BC ++}}\\
\midrule
\parbox[t]{1mm}{\multirow{6}{*}{\rotatebox[origin=c]{90}{Walker2d}}}
& Expert & 66.1 &$\pm$ {\color{mine}{22.7}}  & 60.9 &$\pm$ {\color{mine}{23.2}}  & 104.0 &$\pm$ {\color{mine}{6.7}}  & 75.1 &$\pm$ {\color{mine}{15.7}}  & \textbf{104.5} &$\pm$ {\color{mine}{5.0}}  & 101.3 &$\pm$ {\color{mine}{12.8}}   & 75.6 &$\pm$ {\color{mine}{42.0}} & 64.3 &$\pm$ {\color{mine}{23.9}} & 105.6 &$\pm${\color{mine}{ 3.7 }} & 97.2 &$\pm$ {\color{mine}{8.3}}  & 102.9 &$\pm$ {\color{mine}{4.3}} \\ 
& Expert-medium & 11.3 &$\pm$ {\color{mine}{8.0}}  & 7.4 &$\pm$ {\color{mine}{18.7}}  & 102.4 &$\pm$ {\color{mine}{13.0}}  & 56.1 &$\pm$ {\color{mine}{11.7}}  & 101.6 &$\pm$ {\color{mine}{10.4}}  & 11.6 &$\pm$ {\color{mine}{28.8}}  & 103.3 &$\pm$ {\color{mine}{5.3}}  & 14.8 &$\pm$ {\color{mine}{9.5}} & \textbf{105.1} &$\pm${\color{mine}{ 4.7 }} & 74.1 &$\pm$ {\color{mine}{9.0}}  & 104.3 &$\pm$ {\color{mine}{6.7}} \\ 
& ER1 & 7.1 &$\pm$ {\color{mine}{15.5}}  & 2.6 &$\pm$ {\color{mine}{11.8}}  & 100.8 &$\pm$ {\color{mine}{10.8}}  & 68.9 &$\pm$ {\color{mine}{13.5}}  & 98.8 &$\pm$ {\color{mine}{20.3}}  & 13.6 &$\pm$ {\color{mine}{26.8}}  & 100.0 &$\pm$ {\color{mine}{16.0}}  & 6.6 &$\pm$ {\color{mine}{14.4}} & \textbf{105.2} &$\pm${\color{mine}{ 3.6 }} & 94.5 &$\pm$ {\color{mine}{9.5}}  & 104.1 &$\pm$ {\color{mine}{5.2}} \\ 
& ER3 & 0.8 &$\pm$ {\color{mine}{0.1}}  & 0.9 &$\pm$ {\color{mine}{0.2}}  & 97.8 &$\pm$ {\color{mine}{13.4}}  & 2.2 &$\pm$ {\color{mine}{5.0}}  & 14.2 &$\pm$ {\color{mine}{21.7}}  & 3.5 &$\pm$ {\color{mine}{18.1}}  & 95.0 &$\pm$ {\color{mine}{25.8}} & 9.9 &$\pm$ {\color{mine}{19.9}}  & 102.9 &$\pm${\color{mine}{ 9.0 }} & 95.1 &$\pm$ {\color{mine}{8.1}}  & \textbf{104.3} &$\pm$ {\color{mine}{3.1}} \\ 
& ER5 & 1.0 &$\pm$ {\color{mine}{0.3}}  & 1.2 &$\pm$ {\color{mine}{0.4}}  & 93.2 &$\pm$ {\color{mine}{21.9}}  & 5.2 &$\pm$ {\color{mine}{5.5}}  & 8.6 &$\pm$ {\color{mine}{16.3}}  & 19.9 &$\pm$ {\color{mine}{23.9}}  & 82.5 &$\pm$ {\color{mine}{26.0}}  & 4.0 &$\pm$ {\color{mine}{10.4}} & 92.2 &$\pm${\color{mine}{ 12.6 }} & 87.0 &$\pm$ {\color{mine}{11.4}}  & \textbf{104.4} &$\pm$ {\color{mine}{5.2}} \\ 
& ER7 & 3.4 &$\pm$ {\color{mine}{8.0}}  & 1.9 &$\pm$ {\color{mine}{2.3}}  & 77.0 &$\pm$ {\color{mine}{28.3}}  & -0.2 &$\pm$ {\color{mine}{0.7}}  & 19.6 &$\pm$ {\color{mine}{23.0}}  & 9.6 &$\pm$ {\color{mine}{24.0}}  & 69.3 &$\pm$ {\color{mine}{33.7}} & 2.2 &$\pm$ {\color{mine}{4.0}}  & 67.6 &$\pm${\color{mine}{ 29.5 }} & 73.1 &$\pm$ {\color{mine}{12.4}}  & \textbf{100.2} &$\pm$ {\color{mine}{9.0}} \\ 

\midrule
\parbox[t]{1mm}{\multirow{6}{*}{\rotatebox[origin=c]{90}{Hopper}}}
& Expert & 111.7 &$\pm$ {\color{mine}{1.7}}  & 111.8 &$\pm$ {\color{mine}{1.7}}  & 111.7 &$\pm$ {\color{mine}{2.3}}  & 61.5 &$\pm$ {\color{mine}{54.3}}  & 112.2 &$\pm$ {\color{mine}{0.2}}  & \textbf{112.3} &$\pm$ {\color{mine}{0.3}}  & 112.2 &$\pm$ {\color{mine}{0.7}}  & 106.8 &$\pm$ {\color{mine}{10.8}} & \textbf{112.5} &$\pm${\color{mine}{ 0.2 }} & 111.4 &$\pm$ {\color{mine}{2.7}}  & 112.3 &$\pm$ {\color{mine}{0.2}} \\ 
& Expert-medium & 77.0 &$\pm$ {\color{mine}{38.6}}  & 1.7 &$\pm$ {\color{mine}{0.7}}  & \textbf{112.1} &$\pm$ {\color{mine}{0.3}}  & 85.1 &$\pm$ {\color{mine}{20.9}}  & 112.0 &$\pm$ {\color{mine}{0.4}}  & 1.4 &$\pm$ {\color{mine}{0.6}}  & 112.3 &$\pm$ {\color{mine}{0.3}} & 70.8 &$\pm$ {\color{mine}{33.3}}  & \textbf{112.5} &$\pm${\color{mine}{ 0.4 }} & 110.3 &$\pm$ {\color{mine}{3.8}}  & 112.1 &$\pm$ {\color{mine}{0.3}} \\ 
& ER1 & 106.6 &$\pm$ {\color{mine}{17.0}}  & 104.5 &$\pm$ {\color{mine}{20.1}}  & 112.1 &$\pm$ {\color{mine}{0.4}}  & 104.4 &$\pm$ {\color{mine}{12.8}}  & 112.2 &$\pm$ {\color{mine}{0.2}}  & 11.2 &$\pm$ {\color{mine}{4.9}}  & 112.3 &$\pm$ {\color{mine}{0.2}} & 91.5 &$\pm$ {\color{mine}{23.6}} & \textbf{112.6} &$\pm${\color{mine}{ 0.1 }} & 111.6 &$\pm$ {\color{mine}{3.6}}  & 112.3 &$\pm$ {\color{mine}{0.3}} \\ 
& ER3 & 25.8 &$\pm$ {\color{mine}{25.9}}  & 34.8 &$\pm$ {\color{mine}{30.8}}  & 111.2 &$\pm$ {\color{mine}{2.8}}  & 82.0 &$\pm$ {\color{mine}{13.8}}  & 112.1 &$\pm$ {\color{mine}{0.2}}  & 2.3 &$\pm$ {\color{mine}{1.5}}  & 112.1 &$\pm$ {\color{mine}{0.7}} & 9.9 &$\pm$ {\color{mine}{0.3}}  & \textbf{112.4} &$\pm${\color{mine}{ 0.2 }} & 104.8 &$\pm$ {\color{mine}{11.1}}  & 112.2 &$\pm$ {\color{mine}{0.2}} \\ 
& ER5 & 15.8 &$\pm$ {\color{mine}{20.8}}  & 10.6 &$\pm$ {\color{mine}{4.5}}  & 112.0 &$\pm$ {\color{mine}{1.8}}  & 27.1 &$\pm$ {\color{mine}{11.1}}  & \textbf{112.2} &$\pm$ {\color{mine}{0.2}}  & 15.2 &$\pm$ {\color{mine}{20.9}}  & 112.2 &$\pm$ {\color{mine}{0.2}} & 9.9 &$\pm$ {\color{mine}{0.2}} & 111.6 &$\pm${\color{mine}{ 2.4 }} & 92.0 &$\pm$ {\color{mine}{12.0}}  & \textbf{112.2} &$\pm$ {\color{mine}{0.3}} \\ 
& ER7 & 9.6 &$\pm$ {\color{mine}{0.2}}  & 10.0 &$\pm$ {\color{mine}{2.4}}  & 17.9 &$\pm$ {\color{mine}{21.0}}  & 10.0 &$\pm$ {\color{mine}{0.1}}  & 112.0 &$\pm$ {\color{mine}{0.7}}  & 0.6 &$\pm$ {\color{mine}{0.0}}  & 112.1 &$\pm$ {\color{mine}{0.8}}  & 9.7 &$\pm$ {\color{mine}{0.2}} & \textbf{112.5} &$\pm${\color{mine}{ 0.1 }} & 45.8 &$\pm$ {\color{mine}{40.4}}  & 112.1 &$\pm$ {\color{mine}{0.2}} \\

\midrule
\parbox[t]{1mm}{\multirow{6}{*}{\rotatebox[origin=c]{90}{Halfcheetah}}}
& Expert & 105.8 &$\pm$ {\color{mine}{2.4}}  & 105.5 &$\pm$ {\color{mine}{2.6}}  & 94.7 &$\pm$ {\color{mine}{7.3}}  & 103.8 &$\pm$ {\color{mine}{6.0}}  & 105.3 &$\pm$ {\color{mine}{4.3}}  & 105.6 &$\pm$ {\color{mine}{2.7}}  & 106.5 &$\pm$ {\color{mine}{3.5}} & 95.1 &$\pm$ {\color{mine}{10.2}}  & 102.4 &$\pm${\color{mine}{ 3.8 }} & 104.5 &$\pm$ {\color{mine}{3.4}}  & \textbf{105.9} &$\pm$ {\color{mine}{3.4}} \\ 
& Expert-medium & 65.9 &$\pm$ {\color{mine}{19.0}}  & 74.6 &$\pm$ {\color{mine}{30.0}}  & 33.3 &$\pm$ {\color{mine}{10.9}}  & 49.3 &$\pm$ {\color{mine}{9.5}}  & 94.9 &$\pm$ {\color{mine}{6.3}}  & 1.3 &$\pm$ {\color{mine}{1.5}}  & 95.3 &$\pm$ {\color{mine}{9.9}} & 38.0 &$\pm$ {\color{mine}{4.4}} & 81.9 &$\pm${\color{mine}{ 7.3 }} & 91.0 &$\pm$ {\color{mine}{9.3}}  & \textbf{105.3} &$\pm$ {\color{mine}{2.3}} \\ 
& ER1 & 89.6 &$\pm$ {\color{mine}{11.1}}  & 95.3 &$\pm$ {\color{mine}{8.3}}  & 83.0 &$\pm$ {\color{mine}{11.4}}  & 93.9 &$\pm$ {\color{mine}{17.5}}  & 101.5 &$\pm$ {\color{mine}{5.2}}  & 40.7 &$\pm$ {\color{mine}{26.1}}  & 93.3 &$\pm$ {\color{mine}{11.3}} & 63.5 &$\pm$ {\color{mine}{19.3}} & 76.1 &$\pm${\color{mine}{ 9.1 }} & 100.4 &$\pm$ {\color{mine}{7.6}}  & \textbf{105.1} &$\pm$ {\color{mine}{3.9}} \\ 
& ER3 & 66.1 &$\pm$ {\color{mine}{17.6}}  & 58.8 &$\pm$ {\color{mine}{19.8}}  & 62.0 &$\pm$ {\color{mine}{14.4}}  & 82.2 &$\pm$ {\color{mine}{19.4}}  & 98.4 &$\pm$ {\color{mine}{7.2}}  & 37.6 &$\pm$ {\color{mine}{31.8}}  & 67.8 &$\pm$ {\color{mine}{21.0}} & 22.6 &$\pm$ {\color{mine}{17.4}} & 64.2 &$\pm${\color{mine}{ 12.6 }} & 103.0 &$\pm$ {\color{mine}{5.3}}  & \textbf{103.8} &$\pm$ {\color{mine}{4.2}} \\ 
& ER5 & 30.1 &$\pm$ {\color{mine}{15.5}}  & 19.2 &$\pm$ {\color{mine}{12.4}}  & 55.8 &$\pm$ {\color{mine}{11.5}}  & 43.4 &$\pm$ {\color{mine}{20.8}}  & 90.1 &$\pm$ {\color{mine}{9.7}}  & 43.3 &$\pm$ {\color{mine}{30.3}}  & 46.9 &$\pm$ {\color{mine}{17.7}} & 2.3 &$\pm$ {\color{mine}{0.1}} & 53.0 &$\pm${\color{mine}{ 10.1 }} & 100.3 &$\pm$ {\color{mine}{8.3}}  & \textbf{105.2} &$\pm$ {\color{mine}{2.2}} \\ 
& ER7 & 2.5 &$\pm$ {\color{mine}{1.5}}  & 2.4 &$\pm$ {\color{mine}{0.3}}  & 40.2 &$\pm$ {\color{mine}{13.0}}  & 2.3 &$\pm$ {\color{mine}{0.0}}  & 67.6 &$\pm$ {\color{mine}{9.6}}  & 30.9 &$\pm$ {\color{mine}{42.5}}  & 29.0 &$\pm$ {\color{mine}{12.5}} & 2.3 &$\pm$ {\color{mine}{0.0}} & 31.0 &$\pm${\color{mine}{ 10.9 }} & \textbf{101.8} &$\pm$ {\color{mine}{4.6}}  & 99.8 &$\pm$ {\color{mine}{4.7}} \\ 

\midrule
Total & & 796.2 &$\pm$ {\color{mine}{225.9}}  & 704.1 &$\pm$ {\color{mine}{190.1}}  & 1521.2 &$\pm$ {\color{mine}{191.3}}  & 952.1 &$\pm$ {\color{mine}{238.4}}  & 1577.8 &$\pm$ {\color{mine}{140.9}}  & 561.8 &$\pm$ {\color{mine}{297.8}}  & 1637.9 &$\pm$ {\color{mine}{227.3}} & 624.1 &$\pm$ {\color{mine}{201.7}} & 1663.3 & $\pm$ {\color{mine}{124.1}} & 1697.9 &\ (+78.3\%)  & 1918.5 &\  (+21.6\%) \\ 

\bottomrule
\end{tabular}
\end{adjustbox}
\end{table*}
\begin{table*}[!htp]
\centering
\vspace{-0.2cm}
\caption{Evaluation on the D4RL Adroit domain, involves controlling a 24-DoF robotic hand to perform different tasks. EC1 is short for Expert-cloned-10, with cloned trajectories for non-expert behaviors. The highest performing scores are bolded.}\vspace{0.1cm}
\label{tab:performance_on_adriot}
\begin{adjustbox}{width={\textwidth},totalheight={\textheight},keepaspectratio}%
\renewcommand{\arraystretch}{1.2}
\begin{tabular}{ccr@{\hspace{1pt}}lr@{\hspace{1pt}}lr@{\hspace{1pt}}lr@{\hspace{1pt}}lr@{\hspace{1pt}}lr@{\hspace{1pt}}lr@{\hspace{1pt}}lr@{\hspace{1pt}}lr@{\hspace{1pt}}l | r@{\hspace{1pt}}lr@{\hspace{1pt}}l}
\toprule
\textbf{Task}   & \textbf{Setting}  & \multicolumn{2}{c}{\textbf{BC}} & \multicolumn{2}{c}{\textbf{\%BC}} & \multicolumn{2}{c}{\textbf{CQL}} & \multicolumn{2}{c}{\textbf{BEAR-QL}} &  \multicolumn{2}{c}{\textbf{TD3+BC}} & \multicolumn{2}{c}{\textbf{\%TD3+BC}} & \multicolumn{2}{c}{\textbf{Fisher-BRC}} & \multicolumn{2}{c}{\textbf{UWAC}} & \multicolumn{2}{c}{\textbf{IQL}} & \multicolumn{2}{c}{\textbf{BEAR ++}}   & \multicolumn{2}{c}{\textbf{TD3BC ++}}\\
\midrule

\parbox[t]{1mm}{\multirow{4}{*}{\rotatebox[origin=c]{90}{Door}}}
& Expert & 104.6 &$\pm$ {\color{mine}{1.1}}  & 104.6 &$\pm$ {\color{mine}{1.3}}  & 102.9 &$\pm$ {\color{mine}{5.0}}  & 104.8 &$\pm$ {\color{mine}{0.5}}  & 103.7 &$\pm$ {\color{mine}{3.5}}  & 104.4 &$\pm$ {\color{mine}{3.7}} & 49.4 &$\pm$ {\color{mine}{23.5}} & 104.5 &$\pm$ {\color{mine}{1.2}} & \textbf{105.6} &$\pm${\color{mine}{ 1.4 }} & 104.8 &$\pm$ {\color{mine}{0.7}}  & 105.1 &$\pm$ {\color{mine}{0.3}} \\ 

& EC3 & 102.3 &$\pm$ {\color{mine}{14.7}}  & 103.0 &$\pm$ {\color{mine}{8.9}}  & 101.9 &$\pm$ {\color{mine}{3.1}}  & 104.4 &$\pm$ {\color{mine}{1.0}}  & 0.0 &$\pm$ {\color{mine}{0.0}}  & 35.3 &$\pm$ {\color{mine}{50.0}} & -0.0 &$\pm$ {\color{mine}{0.1}} & 104.0 &$\pm$ {\color{mine}{1.4}} & 104.3 &$\pm${\color{mine}{ 2.5 }} & 104.5 &$\pm$ {\color{mine}{0.8}}  & \textbf{105.2} &$\pm$ {\color{mine}{0.6}} \\

& EC5 & 103.5 &$\pm$ {\color{mine}{1.9}}  & 86.1 &$\pm$ {\color{mine}{32.3}}  & -0.2 &$\pm$ {\color{mine}{0.0}}  & 82.4 &$\pm$ {\color{mine}{20.9}}  & -0.1 &$\pm$ {\color{mine}{0.0}}  & 0.1 &$\pm$ {\color{mine}{0.5}} & -0.0 &$\pm$ {\color{mine}{0.1}} & 101.9 &$\pm$ {\color{mine}{3.0}} & 104.2 &$\pm${\color{mine}{ 2.8 }} & \textbf{104.6} &$\pm$ {\color{mine}{0.8}}  & 104.4 &$\pm$ {\color{mine}{1.8}} \\

& EC7 & 52.2 &$\pm$ {\color{mine}{39.2}}  & 43.1 &$\pm$ {\color{mine}{40.0}}  & -0.2 &$\pm$ {\color{mine}{0.1}}  & -0.2 &$\pm$ {\color{mine}{0.1}}  & 0.0 &$\pm$ {\color{mine}{0.0}}  & -0.2 &$\pm$ {\color{mine}{0.0}}  & -0.0 &$\pm$ {\color{mine}{0.1}} & 92.3 &$\pm$ {\color{mine}{9.9}} & 104.3 &$\pm${\color{mine}{ 2.5 }} & 103.0 &$\pm$ {\color{mine}{1.4}}  & \textbf{104.5} &$\pm$ {\color{mine}{1.4}} \\

\midrule
\parbox[t]{1mm}{\multirow{4}{*}{\rotatebox[origin=c]{90}{Hammer}}}
& Expert & 126.6 &$\pm$ {\color{mine}{0.5}}  & 126.5 &$\pm$ {\color{mine}{0.5}}  & \multicolumn{2}{c}{-}  & 126.9 &$\pm$ {\color{mine}{0.3}}  & 127.8 &$\pm$ {\color{mine}{0.6}}  & \textbf{128.0} &$\pm$ {\color{mine}{0.4}} & 35.9 &$\pm$ {\color{mine}{33.9}} & 126.2 &$\pm$ {\color{mine}{0.6}} & 119.7 &$\pm${\color{mine}{ 12.5 }} & 126.9 &$\pm$ {\color{mine}{0.5}}  & 126.8 &$\pm$ {\color{mine}{0.5}} \\

& EC3 & 126.9 &$\pm$ {\color{mine}{0.7}}  & 127.1 &$\pm$ {\color{mine}{0.4}}  & \multicolumn{2}{c}{-}  & 84.7 &$\pm$ {\color{mine}{59.7}}  & \textbf{128.0} &$\pm$ {\color{mine}{0.4}}  & 34.5 &$\pm$ {\color{mine}{41.2}} & 0.2 &$\pm$ {\color{mine}{0.1}}& 126.6 &$\pm$ {\color{mine}{0.6}}  & 124.9 &$\pm${\color{mine}{ 5.8 }} & 126.7 &$\pm$ {\color{mine}{0.6}}  & 126.9 &$\pm$ {\color{mine}{0.5}} \\

& EC5 & 120.4 &$\pm$ {\color{mine}{18.2}}  & 115.5 &$\pm$ {\color{mine}{25.9}}  & \multicolumn{2}{c}{-}  & 90 &$\pm$ {\color{mine}{42.0}}  & \textbf{128.4} &$\pm$ {\color{mine}{0.7}}  & 51.1 &$\pm$ {\color{mine}{46.5}} & 0.2 &$\pm$ {\color{mine}{0.0}} & 125.4 &$\pm$ {\color{mine}{4.1}} & 126.8 &$\pm${\color{mine}{ 2.4 }} & 127.0 &$\pm$ {\color{mine}{0.4}}  & 127.1 &$\pm$ {\color{mine}{0.5}} \\

& EC7 & 73.7 &$\pm$ {\color{mine}{28.0}}  & 62.0 &$\pm$ {\color{mine}{28.1}}  & \multicolumn{2}{c}{-}  & 21.0 &$\pm$ {\color{mine}{46.5}}  & 0.8 &$\pm$ {\color{mine}{0.6}}  & 0.1 &$\pm$ {\color{mine}{0.1}} & 0.3 &$\pm$ {\color{mine}{0.2}} & 107.9 &$\pm$ {\color{mine}{19.5}} & 127.6 &$\pm${\color{mine}{ 0.6 }} & 127.0 &$\pm$ {\color{mine}{0.7}}  & \textbf{127.9} &$\pm$ {\color{mine}{1.9}} \\

\midrule
\parbox[t]{1mm}{\multirow{4}{*}{\rotatebox[origin=c]{90}{Pen}}}
& Expert & 157.5 &$\pm$ {\color{mine}{5.4}}  & \textbf{157.8} &$\pm$ {\color{mine}{5.4}}  & 94.7 &$\pm$ {\color{mine}{25.8}}  & 155.5 &$\pm$ {\color{mine}{2.0}}  & 132.5 &$\pm$ {\color{mine}{26.3}}  & 121.3 &$\pm$ {\color{mine}{34.5}}& \multicolumn{2}{c}{-} & 155.1 &$\pm$ {\color{mine}{2.6}} & 155.8 &$\pm${\color{mine}{ 5.4 }} & 155.0 &$\pm$ {\color{mine}{2.3}}  & 150.3 &$\pm$ {\color{mine}{9.1}} \\

& EC3 & 145.8 &$\pm$ {\color{mine}{24.4}}  & 148.1 &$\pm$ {\color{mine}{20.5}}  & 66.1 &$\pm$ {\color{mine}{50.0}}  & -3.7 &$\pm$ {\color{mine}{0.4}}  & 100.4 &$\pm$ {\color{mine}{10.3}}  & 100.2 &$\pm$ {\color{mine}{26.3}} & \multicolumn{2}{c}{-} & 154.5 &$\pm$ {\color{mine}{2.4}} & \textbf{156.1} &$\pm${\color{mine}{ 5.1 }} & 154.3 &$\pm$ {\color{mine}{1.8}}  & 128.9 &$\pm$ {\color{mine}{42.3}} \\

& EC5 & 67.9 &$\pm$ {\color{mine}{38.1}}  & 89.2 &$\pm$ {\color{mine}{31.7}}  & -1.6 &$\pm$ {\color{mine}{2.1}}  & -2.6 &$\pm$ {\color{mine}{0.2}}  & 67.1 &$\pm$ {\color{mine}{37.1}}  & 96.2 &$\pm$ {\color{mine}{20.5}} & \multicolumn{2}{c}{-} & 152.6 &$\pm$ {\color{mine}{2.2}} & 154.3 &$\pm${\color{mine}{ 6.1 }} & \textbf{153.9} &$\pm$ {\color{mine}{2.8}}  & 141.6 &$\pm$ {\color{mine}{17.8}} \\

& EC7 & 61.8 &$\pm$ {\color{mine}{33.7}}  & 59.2 &$\pm$ {\color{mine}{22.3}}  & -1.6 &$\pm$ {\color{mine}{2.4}}  & -2.4 &$\pm$ {\color{mine}{0.1}}  & 65.5 &$\pm$ {\color{mine}{25.5}}  & -1.4 &$\pm$ {\color{mine}{1.5}} & \multicolumn{2}{c}{-} & 59.1 &$\pm$ {\color{mine}{15.6}} & \textbf{154.6} &$\pm${\color{mine}{ 6.3 }} & 63.8 &$\pm$ {\color{mine}{15.5}}  & 101.5 &$\pm$ {\color{mine}{21.2}} \\

\midrule
\parbox[t]{1mm}{\multirow{4}{*}{\rotatebox[origin=c]{90}{Relocate}}}
& Expert & 102.3 &$\pm$ {\color{mine}{3.6}}  & 102.9 &$\pm$ {\color{mine}{3.3}}  & \multicolumn{2}{c}{-}  & \textbf{105.2} &$\pm$ {\color{mine}{1.5}}  & \textbf{105.2} &$\pm$ {\color{mine}{2.3}}  & \textbf{105.2} &$\pm$ {\color{mine}{2.3}} & 3.9 &$\pm$ {\color{mine}{6.1}} & 105.1 &$\pm$ {\color{mine}{2.8}} & 104.9 &$\pm${\color{mine}{ 4.4 }} & \textbf{105.2} &$\pm$ {\color{mine}{2.5}}  & 103.5 &$\pm$ {\color{mine}{4.1}} \\
& EC3 & 103.2 &$\pm$ {\color{mine}{3.8}}  & 100.9 &$\pm$ {\color{mine}{5.0}}  & \multicolumn{2}{c}{-}  & -0.3 &$\pm$ {\color{mine}{0.0}}  & 103.9 &$\pm$ {\color{mine}{3.3}}  & 71.5 &$\pm$ {\color{mine}{37.9}} & -0.0 &$\pm$ {\color{mine}{0.1}} & 104.1 &$\pm$ {\color{mine}{3.7}} & \textbf{107.1} &$\pm${\color{mine}{ 2.7 }} & 105.9 &$\pm$ {\color{mine}{1.4}}  & 104.4 &$\pm$ {\color{mine}{2.5}} \\
& EC5 & 82.1 &$\pm$ {\color{mine}{23.5}}  & 89.0 &$\pm$ {\color{mine}{18.5}}  & \multicolumn{2}{c}{-}  & -0.3 &$\pm$ {\color{mine}{0.0}}  & 97.5 &$\pm$ {\color{mine}{9.0}}  & 80.4 &$\pm$ {\color{mine}{15.3}} & -0.0 &$\pm$ {\color{mine}{0.2}} & 103.2 &$\pm$ {\color{mine}{3.5}} & \textbf{106.2} &$\pm${\color{mine}{ 3.7 }} & 105.4 &$\pm$ {\color{mine}{1.7}}  & 103.4 &$\pm$ {\color{mine}{2.4}} \\
& EC7 & 40.1 &$\pm$ {\color{mine}{27.4}}  & 35.4 &$\pm$ {\color{mine}{29.4}}  & \multicolumn{2}{c}{-}  & -0.3 &$\pm$ {\color{mine}{0.0}}  & 27.4 &$\pm$ {\color{mine}{34.0}}  & 0.0 &$\pm$ {\color{mine}{0.1}} & -0.0 &$\pm$ {\color{mine}{0.1}} & 74.9 &$\pm$ {\color{mine}{9.7}} & \textbf{107.0} &$\pm${\color{mine}{ 3.1 }} & 102.4 &$\pm$ {\color{mine}{3.0}}  & 99.2 &$\pm$ {\color{mine}{6.7}} \\
\midrule
Total & & 1571.1 &$\pm$ {\color{mine}{264.2}}  & 1550.4 &$\pm$ {\color{mine}{273.5}}  & \multicolumn{2}{c}{-} & 865.0 &$\pm$ {\color{mine}{175.2}}  & 1188.1 &$\pm$ {\color{mine}{153.7}}  & 926.8 &$\pm$ {\color{mine}{281.0}} & \multicolumn{2}{c}{-} & 1797.5 & $\pm$ {\color{mine}{82.9}}  & 1963.1 & $\pm${\color{mine}{68.7}} & 1870.3 & \ (+116.2\%)  & 1860.8 & \ (+56.6\%) \\
\bottomrule
\end{tabular}
\end{adjustbox}
\end{table*}

\subsection{Setup}

\vspace{-0.2cm}
\paragraph{Datasets.}
We consider three types of contaminated datasets, expert-medium, expert-cloned and expert-random (see Appendix \ref{app:the_modified_dataset}). For performance on original tasks, please refer to Appendix \ref{app:evaluation_on_mujoco}.

\vspace{-0.2cm}
\paragraph{Evaluation.}
We train each algorithm for 1 million training time steps, evaluate them every 5000 time steps and finally report the mean and standard deviation of the normalized score \cite{fu2020d4rl} over the final 500 episodes (10 trajectories, 10 evaluations and 5 seeds). Please note that 5-seed evaluation is a common setting for offline RL evaluation~\cite{kumar2019stabilizing, wu2019behavior, fujimoto2021minimalist, ma2021conservative, wu2021uncertainty, sinha2022s4rl}. 

\vspace{-0.2cm}
\paragraph{Baselines.}  We compare TD3BC++ and BEAR++ with BC \cite{pomerleau1991efficient}, CQL \cite{kumar2020conservative}, IQL \cite{kostrikov2021offlineiql}, UWAC~\cite{wu2021uncertainty}, Fihser-BRC\cite{kostrikov2021offline} and original BEAR-QL \cite{kumar2019stabilizing}, TD3+BC \cite{fujimoto2021minimalist}. We also examine percentile BC and percentile TD3+BC, i.e., run BC and TD3+BC on the top $X\%$ of transitions with higher immediate rewards. We set $X\%$ to the percentage of expert transitions in each dataset.

\subsection{Results and discussion}

\vspace{-0.2cm}
\paragraph{Performance degradation (Table \ref{tab:performance_on_mujoco}).}
The proposed TD3BC++ and BEAR++ show resistance to performance degradation, i.e., agents trained on expert-medium datasets perform as well as on expert datasets, for all 3 mujoco gym tasks.

\vspace{-0.2cm}
\paragraph{Catastrophic failure (Table \ref{tab:performance_on_mujoco} and \ref{tab:performance_on_adriot}).}
BEAR-QL and TD3+BC suffer from catastrophic failure when learning on datasets contain low-level trajectories, e.g., expert-random and expert-cloned datasets. Fortunately, the proposed methods alleviate the catastrophic failure issues for all 7 tasks. And they could even help TD3+BC perform as well on ER7 as on the expert datasets, in 6 of 7 tasks.

\vspace{-0.2cm}
\paragraph{Further penalization on OOD actions (Figure \ref{fig:tech_gp_save_the_failed_colseness_constraint}).}
Recall that the failed closeness constraint on non-expert decisions produces OOD policy actions that differ from the dataset actions. And the proposed gradient penalty successfully recovers the Q-function by penalizing the unstable sharp Q gradients. In light of this, we suspect that GP could also contribute to reducing the strength of the required closeness constraint for policy constraint based offline RL. 

\begin{figure*}[!ht]
    \centering
    \includegraphics[width=\linewidth]{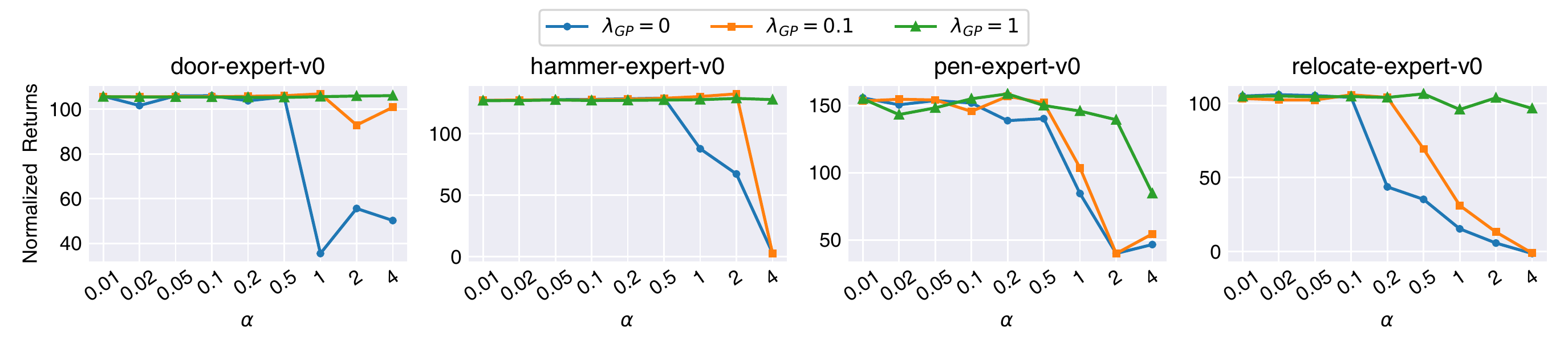}\vspace{-0.4cm}
    \caption{Gradient penalty alleviates the dependence on policy constraints. We run TD3+BC plus different strengths of gradient penalty ($\lambda_{GP}=0,0.1,1$) and different strengths of BC term (X-axis, $\alpha=0$ for entire BC and $\alpha=4$ for RL) on Adroit tasks.}\vspace{-0.2cm}
    \label{fig:tech_gp_save_the_failed_colseness_constraint} 
\end{figure*}
\begin{figure*}[!ht]
    \centering
    \includegraphics[width=\linewidth]{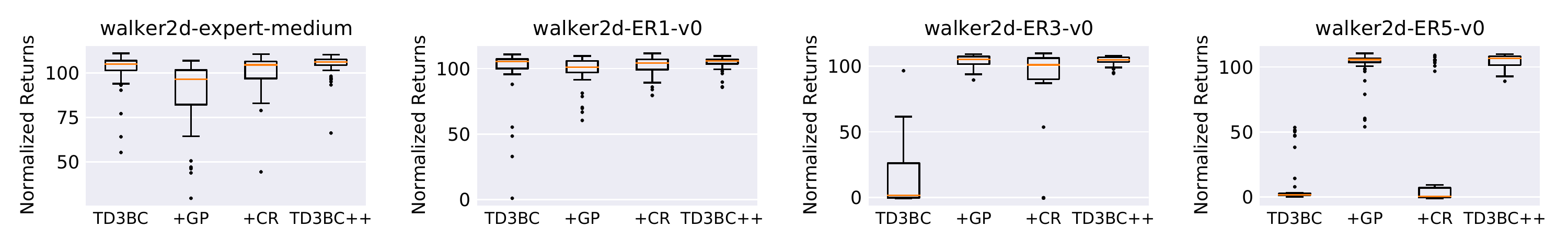}\vspace{-0.4cm}
    \caption{Ablation stduy. Box plot. We run original TD3+BC, TD3+BC with gradient penalty (+GP), TD3+BC with critic weighted constraint relaxation (+CR), and TD3BC++ on walker2d tasks.}\vspace{-0.2cm}
    \label{fig:ablation}
\end{figure*}

To investigate this, we run TD3+BC plus GP, changing the hyperparameter $\alpha$ to control the strength of BC term (Equation \ref{eq:td3_bc}). Note that with $\alpha=1$, the agent prefers imitation while with $\alpha=4$ for RL. 
Figure \ref{fig:tech_gp_save_the_failed_colseness_constraint} demonstrates that GP alleviates the dependence on policy constraints for TD3+BC, and thus may save it from degrading to behavioral cloning. 

It is not surprising that GP can reduce the dependence on the strength of constraints for policy constraint based offline RL algorithms. Because the failed closeness constraint can be caused not only by the policy improvement on contaminated datasets but also by poor closeness metrics.

\vspace{-0.2cm}
\paragraph{Ablation study (Figure \ref{fig:ablation}).} 

We ablate the effects of the two proposed techniques when applied individually. For datasets contain many low-level demonstrations (ER3 and ER5 settings), the gradient penalty stabilizes Q-values and prevents complete catastrophic failure, and the constraint relaxation with polished Q-weights brings performance back up to the expert level. 
For datasets that contain a medium level or a low proportion of random demonstrations (EM and ER1 settings), catastrophic failures do not occur. In this case, constraint relaxtion alone is effective, and it performs better in conjunction with a polished Q-function.

\vspace{-0.2cm}
\paragraph{Comparison with the na\"ive solution (Table \ref{tab:performance_on_mujoco} and \ref{tab:performance_on_adriot}).} $\%$BC and $\%$TD3+BC show slight resistance to performance degradation and catastrophic failures. We also set $X$ to $X\pm10$ and find it performs worse.

\vspace{-0.2cm}
\paragraph{Gradient penalty w.r.t. input states.} 
We also investigate the effect of gradient penalty w.r.t. input states. However, we discover experimentally that it performs much worse.

\vspace{-0.2cm}
\paragraph{Computational cost comparison.} 
We train TD3+BC and TD3BC++ agents for 1 million time steps. The wall clock time of TD3+BC is 160m, and 173m for TD3BC++, indicating that the two techniques proposed in this paper are light and efficient plugins for policy constrained offline RL.

\section{Related work}

\vspace{-0.2cm}
\paragraph{Policy constrained offline RL.} 
One main approach for offline RL is to enforce the learned policy stay close to the behavior policy, involved with various closeness measurements such as KL-divergence~\cite{jaques2019way}, maximum mean discrepancy (MMD)~\cite{kumar2019stabilizing}, Wasserstein distance~\cite{wu2019behavior}, Fisher divergence~\cite{kostrikov2021offline} and even Euclidean distance~\cite{fujimoto2021minimalist}. 
Closeness constraints could help avoid OOD actions. However, when training on contaminated datasets with non-expert demonstrations, a common setting in real-world applications, these methods show performance degradation and even catastrophic failure in our observation. The proposed two techniques serve to mitigate such issues.

\vspace{-0.2cm}
\paragraph{Value regularization offline RL.} 
Another offline RL approach is modifying the Q-values to prevent overestimation on OOD actions. This can be achieved by directly penalizing the Q-values of OOD actions in the regression target, e.g., CQL~\cite{kumar2020conservative}, IQL~\cite{kostrikov2021offlineiql}, SAC-N~\cite{an2021uncertainty} or discounting them with uncertainty measurements e.g., UWAC~\cite{wu2021uncertainty}, EDAC~\cite{an2021uncertainty}, PBRL~\cite{bai2022pessimistic}, RORL~\cite{yang2022rorl}. The proposed GP aims not to prevent the OOD actions but rather to minimize the negative impact (unstable Q-function gradients) of OOD actions caused by the failed closeness constraint on non-expert states.

\vspace{-0.2cm}
\paragraph{Lipschitzness in reinforcement learning.}
Our method penalizes the sharp gradient derived from the critic, which is similar to enforcing the learned Q-function to be locally Lipschitz-continuous. Lipschitzness is often used for stabilizing generative adversarial network (GAN) training. It can be achieved by gradient penalty~\cite{thanh2019improving}, spectral normalization~\cite{miyato2018spectral}, gradient normalization~\cite{wu2021gradient}, etc. In online RL, \cite{gogianu2021spectral} use spectral normalization to better the optimization dynamics of the Bellman backups. \cite{lecarpentier2020lipschitz} utilize Lipschitz continuity between MDPs to transfer knowledge for lifelong RL tasks. 
\cite{memarian2021robust} promotes a local Lipschitz discriminator for robust generative adversarial imitation learning (GAIL) algorithms. Our method aims to minimize the effects of the non-expert trajectories, which carries a different motivation.

\vspace{-0.2cm}
\paragraph{Learning from non-expert trajectories.}
This work focus on the influence of non-expert trajectories in the offline RL setting. 
Similarly, \cite{zhang2021reducing} proposes an algorithm to address this issue, assuming clustering methods can recognize transitions from different behavior policies. 
Besides, \cite{zhang2021corruption} consider the task of training policy from datasets with adversarial corruptions. Our method does not rely on such assumptions. In addition, \cite{nair2020accelerating} proposes an advantage-based weighting to ditinguish expert and non-expert data, which is superior to our CR technique. We leave this issue for further work.

Learning from non-expert data is also a key challenge in imitation learning. 
Methods in this topic can be mainly divided into two types. Ranking-based methods learn a policy from demonstrations annotated with rankings~\cite{akrour2011preference, brown2019extrapolating, brown2020better, chen2020learning}. Confidence-based methods construct or learn a confidence value function describing the quality of demonstrations and then reweight training samples for imitation~\cite{wang2018exponentially, wu2019imitation, zhu2020learning, tangkaratt2020variational, sasaki2020behavioral, cao2021learning, zhang2021confidence, wang2021learning}. Our method utilizes the learned Q-function to indicate the optimality of the transition, which is close to the confidence-based methods. 

\section{Conclusion}

By analysising the learning behavoirs on dataset generated by multiple distinct behavior policies, we identify two approaches in which non-expert trajectories inhibit policy constrained offline RL: 1) the harmful closeness constraint towards non-expert actions on overlaped states and 2) the failed closeness constraint on non-expert states that causing OOD actions. The proposed CR and GP techniques correspond to these two pathways, respectively, and their effectiveness is empirically evaluated on expert-medium, expert-cloned, and expert-random settings. 

The proposed two plugins together extend the applicability of the policy constraint based offline RL to contaminated datasets. 
Beyond this, the proposed gradient penalty, can help mitigate the negative impacts of OOD actions when the policy constraint fails (on contaminated datasets) or when the constraint has to be weak (to improve over the behavior policies). We hope this finding would encourage researchers to tackle offline reinforcement learning differently than regularizing the Q-values or constraining the learned policy.

\newpage
{\footnotesize
\bibliographystyle{abbrv}
\bibliography{reference}

\begin{thebibliography}{10}

\bibitem{akrour2011preference}
R.~Akrour, M.~Schoenauer, and M.~Sebag.
\newblock Preference-based policy learning.
\newblock In {\em Joint European Conference on Machine Learning and Knowledge
  Discovery in Databases}, pages 12--27. Springer, 2011.

\bibitem{an2021uncertainty}
G.~An, S.~Moon, J.-H. Kim, and H.~O. Song.
\newblock Uncertainty-based offline reinforcement learning with diversified
  {Q}-ensemble.
\newblock {\em arXiv preprint arXiv:2110.01548}, 2021.

\bibitem{kumar2019stabilizing}
K.~Aviral, F.~Justin, S.~Matthew, T.~George, and L.~Sergey.
\newblock Stabilizing off-policy {Q}-learning via bootstrapping error
  reduction.
\newblock In {\em Advances in Neural Information Processing Systems}, pages
  11761--11771, 2019.

\bibitem{bai2022pessimistic}
C.~Bai, L.~Wang, Z.~Yang, Z.~Deng, A.~Garg, P.~Liu, and Z.~Wang.
\newblock Pessimistic bootstrapping for uncertainty-driven offline
  reinforcement learning.
\newblock {\em arXiv preprint arXiv:2202.11566}, 2022.

\bibitem{brown2019extrapolating}
D.~Brown, W.~Goo, P.~Nagarajan, and S.~Niekum.
\newblock Extrapolating beyond suboptimal demonstrations via inverse
  reinforcement learning from observations.
\newblock In {\em International Conference on Machine Learning}, pages
  783--792. PMLR, 2019.

\bibitem{brown2020better}
D.~S. Brown, W.~Goo, and S.~Niekum.
\newblock Better-than-demonstrator imitation learning via automatically-ranked
  demonstrations.
\newblock In {\em Conference on robot learning}, pages 330--359. PMLR, 2020.

\bibitem{cao2021learning}
Z.~Cao and D.~Sadigh.
\newblock Learning from imperfect demonstrations from agents with varying
  dynamics.
\newblock {\em IEEE Robotics and Automation Letters}, 6(3):5231--5238, 2021.

\bibitem{chen2020learning}
L.~Chen, R.~Paleja, and M.~Gombolay.
\newblock Learning from suboptimal demonstration via self-supervised reward
  regression.
\newblock In {\em Conference on Robot Learning}, pages 1262--1277. {PMLR},
  2020.

\bibitem{fu2020d4rl}
J.~Fu, A.~Kumar, O.~Nachum, G.~Tucker, and S.~Levine.
\newblock D4rl: Datasets for deep data-driven reinforcement learning.
\newblock {\em arXiv preprint arXiv:2004.07219}, 2020.

\bibitem{fu2019diagnosing}
J.~Fu, A.~Kumar, M.~Soh, and S.~Levine.
\newblock Diagnosing bottlenecks in deep {Q}-learning algorithms.
\newblock In {\em International Conference on Machine Learning}, pages
  2021--2030. PMLR, 2019.

\bibitem{fujimoto2021minimalist}
S.~Fujimoto and S.~S. Gu.
\newblock A minimalist approach to offline reinforcement learning.
\newblock {\em Advances in Neural Information Processing Systems}, 34, 2021.

\bibitem{fujimoto2018addressing}
S.~Fujimoto, H.~Hoof, and D.~Meger.
\newblock Addressing function approximation error in actor-critic methods.
\newblock In {\em International Conference on Machine Learning}, pages
  1587--1596. PMLR, 2018.

\bibitem{fujimoto2019off}
S.~Fujimoto, D.~Meger, and D.~Precup.
\newblock Off-policy deep reinforcement learning without exploration.
\newblock In {\em International Conference on Machine Learning}, pages
  2052--2062. PMLR, 2019.

\bibitem{gogianu2021spectral}
F.~Gogianu, T.~Berariu, M.~Rosca, C.~Clopath, L.~Busoniu, and R.~Pascanu.
\newblock Spectral normalisation for deep reinforcement learning: an
  optimisation perspective.
\newblock In {\em International Conference on Machine Learning}, pages
  3734--3744. {PMLR}, 2021.

\bibitem{gunes2014shilling}
I.~Gunes, C.~Kaleli, A.~Bilge, and H.~Polat.
\newblock Shilling attacks against recommender systems: a comprehensive survey.
\newblock {\em Artificial Intelligence Review}, 42(4):767--799, 2014.

\bibitem{hu2021actor}
Y.~Hu, Z.~Ji, and M.~Telgarsky.
\newblock Actor-critic is implicitly biased towards high entropy optimal
  policies.
\newblock {\em arXiv preprint arXiv:2110.11280}, 2021.

\bibitem{huang2021data}
H.~Huang, J.~Mu, N.~Z. Gong, Q.~Li, B.~Liu, and M.~Xu.
\newblock Data poisoning attacks to deep learning based recommender systems.
\newblock {\em arXiv preprint arXiv:2101.02644}, 2021.

\bibitem{jaques2019way}
N.~Jaques, A.~Ghandeharioun, J.~H. Shen, C.~Ferguson, A.~Lapedriza, N.~Jones,
  S.~Gu, and R.~Picard.
\newblock Way off-policy batch deep reinforcement learning of implicit human
  preferences in dialog.
\newblock {\em arXiv preprint arXiv:1907.00456}, 2019.

\bibitem{kostrikov2021offline}
I.~Kostrikov, R.~Fergus, J.~Tompson, and O.~Nachum.
\newblock Offline reinforcement learning with {F}isher divergence critic
  regularization.
\newblock In {\em International Conference on Machine Learning}, pages
  5774--5783. PMLR, 2021.

\bibitem{kostrikov2021offlineiql}
I.~Kostrikov, A.~Nair, and S.~Levine.
\newblock Offline reinforcement learning with implicit {Q}-learning.
\newblock {\em arXiv preprint arXiv:2110.06169}, 2021.

\bibitem{kumar2020conservative}
A.~Kumar, A.~Zhou, G.~Tucker, and S.~Levine.
\newblock Conservative {Q}-learning for offline reinforcement learning.
\newblock In {\em Advances in Neural Information Processing Systems}, 2020.

\bibitem{lange2012batch}
S.~Lange, T.~Gabel, and M.~Riedmiller.
\newblock Batch reinforcement learning.
\newblock In {\em Reinforcement learning}, pages 45--73. Springer, 2012.

\bibitem{lecarpentier2020lipschitz}
E.~Lecarpentier, D.~Abel, K.~Asadi, Y.~Jinnai, E.~Rachelson, and M.~L. Littman.
\newblock Lipschitz lifelong reinforcement learning.
\newblock In {\em Thirty-Fifth {AAAI} Conference on Artificial Intelligence},
  pages 8270--8278. {AAAI} Press, 2020.

\bibitem{levine2020offline}
S.~Levine, A.~Kumar, G.~Tucker, and J.~Fu.
\newblock Offline reinforcement learning: Tutorial, review, and perspectives on
  open problems.
\newblock {\em arXiv preprint arXiv:2005.01643}, 2020.

\bibitem{ma2021conservative}
Y.~Ma, D.~Jayaraman, and O.~Bastani.
\newblock Conservative offline distributional reinforcement learning.
\newblock {\em Advances in Neural Information Processing Systems}, 34, 2021.

\bibitem{memarian2021robust}
F.~Memarian, A.~Hashemi, S.~Niekum, and U.~Topcu.
\newblock Robust generative adversarial imitation learning via local
  {L}ipschitzness.
\newblock {\em arXiv preprint arXiv:2107.00116}, 2021.

\bibitem{miyato2018spectral}
T.~Miyato, T.~Kataoka, M.~Koyama, and Y.~Yoshida.
\newblock Spectral normalization for generative adversarial networks.
\newblock In {\em International Conference on Learning Representations}, 2018.

\bibitem{munos2008finite}
R.~Munos and C.~Szepesv{\'a}ri.
\newblock Finite-time bounds for fitted value iteration.
\newblock {\em Journal of Machine Learning Research}, 9(5), 2008.

\bibitem{nair2020accelerating}
A.~Nair, M.~Dalal, A.~Gupta, and S.~Levine.
\newblock Accelerating online reinforcement learning with offline datasets.
\newblock {\em arXiv preprint arXiv:2006.09359}, 2020.

\bibitem{pomerleau1991efficient}
D.~A. Pomerleau.
\newblock Efficient training of artificial neural networks for autonomous
  navigation.
\newblock {\em Neural computation}, 3(1):88--97, 1991.

\bibitem{rachelson2010locality}
E.~Rachelson and M.~G. Lagoudakis.
\newblock On the locality of action domination in sequential decision making.
\newblock In {\em International Symposium on Artificial Intelligence and
  Mathematics}, 2010.

\bibitem{sasaki2020behavioral}
F.~Sasaki and R.~Yamashina.
\newblock Behavioral cloning from noisy demonstrations.
\newblock In {\em International Conference on Learning Representations}, 2020.

\bibitem{sinha2022s4rl}
S.~Sinha, A.~Mandlekar, and A.~Garg.
\newblock S4rl: Surprisingly simple self-supervision for offline reinforcement
  learning in robotics.
\newblock In {\em Conference on Robot Learning}, pages 907--917. PMLR, 2022.

\bibitem{sun2019adversarial}
M.~Sun and X.~Ma.
\newblock Adversarial imitation learning from incomplete demonstrations.
\newblock {\em arXiv preprint arXiv:1905.12310}, 2019.

\bibitem{tangkaratt2020variational}
V.~Tangkaratt, B.~Han, M.~E. Khan, and M.~Sugiyama.
\newblock Variational imitation learning with diverse-quality demonstrations.
\newblock In {\em International Conference on Machine Learning}, pages
  9407--9417. PMLR, 2020.

\bibitem{tangkaratt2018improving}
V.~Tangkaratt and M.~Sugiyama.
\newblock Improving generative adversarial imitation learning with non-expert
  demonstrations.
\newblock {\em OpenReview}, 2018.

\bibitem{thanh2019improving}
H.~Thanh-Tung, T.~Tran, and S.~Venkatesh.
\newblock Improving generalization and stability of generative adversarial
  networks.
\newblock In {\em International Conference on Learning Representations}, 2019.

\bibitem{wang2018exponentially}
Q.~Wang, J.~Xiong, L.~Han, P.~Sun, H.~Liu, and T.~Zhang.
\newblock Exponentially weighted imitation learning for batched historical
  data.
\newblock In {\em Advances in Neural Information Processing Systems}, pages
  6291--6300, 2018.

\bibitem{wang2021learning}
Y.~Wang, C.~Xu, B.~Du, and H.~Lee.
\newblock Learning to weight imperfect demonstrations.
\newblock In {\em International Conference on Machine Learning}, pages
  10961--10970. PMLR, 2021.

\bibitem{wu2019behavior}
Y.~Wu, G.~Tucker, and O.~Nachum.
\newblock Behavior regularized offline reinforcement learning.
\newblock {\em arXiv preprint arXiv:1911.11361}, 2019.

\bibitem{wu2021uncertainty}
Y.~Wu, S.~Zhai, N.~Srivastava, J.~Susskind, J.~Zhang, R.~Salakhutdinov, and
  H.~Goh.
\newblock Uncertainty weighted actor-critic for offline reinforcement learning.
\newblock In {\em International Conference on Machine Learning}, pages
  11319--11328. PMLR, 2021.

\bibitem{wu2019imitation}
Y.-H. Wu, N.~Charoenphakdee, H.~Bao, V.~Tangkaratt, and M.~Sugiyama.
\newblock Imitation learning from imperfect demonstration.
\newblock In {\em International Conference on Machine Learning}, pages
  6818--6827. PMLR, 2019.

\bibitem{wu2021gradient}
Y.-L. Wu, H.-H. Shuai, Z.-R. Tam, and H.-Y. Chiu.
\newblock Gradient normalization for generative adversarial networks.
\newblock In {\em International Conference on Computer Vision}, pages
  6373--6382, 2021.

\bibitem{yang2022rorl}
R.~Yang, C.~Bai, X.~Ma, Z.~Wang, C.~Zhang, and L.~Han.
\newblock Rorl: Robust offline reinforcement learning via conservative
  smoothing.
\newblock {\em arXiv preprint arXiv:2206.02829}, 2022.

\bibitem{zhang2021reducing}
H.~Zhang, J.~Shao, Y.~Jiang, S.~He, and X.~Ji.
\newblock Reducing conservativeness oriented offline reinforcement learning.
\newblock {\em arXiv preprint arXiv:2103.00098}, 2021.

\bibitem{zhang2021confidence}
S.~Zhang, Z.~Cao, D.~Sadigh, and Y.~Sui.
\newblock Confidence-aware imitation learning from demonstrations with varying
  optimality.
\newblock {\em arXiv preprint arXiv:2110.14754}, 2021.

\bibitem{zhang2021corruption}
X.~Zhang, Y.~Chen, J.~Zhu, and W.~Sun.
\newblock Corruption-robust offline reinforcement learning.
\newblock {\em arXiv preprint arXiv:2106.06630}, 2021.

\bibitem{zhu2020learning}
Z.~Zhu, K.~Lin, B.~Dai, and J.~Zhou.
\newblock Learning sparse rewarded tasks from sub-optimal demonstrations.
\newblock {\em arXiv preprint arXiv:2004.00530}, 2020.

\end{thebibliography}
}



\clearpage

\newpage
\appendix
\part*{Appendices}
\section{The contaminated D4RL datasets}
\label{app:the_modified_dataset}
We first provide details about the contaminated D4RL datasets to accommodate reproducibility. Then we give evidence to support the description of the different state overlaps in Figure~\ref{fig:state_overlap}. And finally, we provide some perspectives on the proposed contaminated D4RL datasets.

\subsection{Dataset statistics}
\paragraph{The contaminated D4RL mujoco gym datasets.} Each contaminated dataset contains trajectories from two different levels of policies. We use the D4RL medium-expert datasets for expert-medium settings, which are combinations of expert and medium-level trajectories and are about twice the size of the corresponding expert or medium datasets.

We also contaminate the expert demonstrations with random-level trajectories. For example, ER-1 (short for Expert-random-10) represents a dataset constructed by first loading an expert dataset and then replacing the final 10 percent transitions with tuples from random trajectories (the first 10 percent in the corresponding random dataset). 
We provide statistics:
\vspace{-0.3cm}
\begin{table}[!h]
\centering
\caption{Statistics of the contaminated D4RL mujoco gym datasets (expert-random).}
\begin{adjustbox}{width={\textwidth},totalheight={\textheight},keepaspectratio}%
\renewcommand{\arraystretch}{1.1}
\begin{tabular}{cccccc}
\toprule
\textbf{Task} & \textbf{Setting} & \textbf{Total transition} & \textbf{Expert transition} & \textbf{Random transition} & \textbf{Averaged reward} \\ \hline
\multirow{4}*{Hopper}
            & Expert-random-10 & 999,034          & 899,131           & 99,903            & 3.53            \\
            & Expert-random-30 & 999,034          & 699,324           & 299,710           & 3.33            \\
            & Expert-random-50 & 999,034          & 499,517           & 499,517           & 3.13            \\
            & Expert-random-70 & 999,034          & 299,711           & 699,323           & 2.93            \\ \hline
\multirow{4}*{Walker}
            & Expert-random-10 & 999,304          & 899,374           & 99,930            & 4.25            \\
            & Expert-random-30 & 999,304          & 699,513           & 299,791           & 3.33            \\
            & Expert-random-50 & 999,304          & 499,652           & 499,652           & 2.39            \\
            & Expert-random-70 & 999,304          & 299,792           & 699,512           & 1.45            \\ \hline
\multirow{4}*{Halfcheetah}
            & Expert-random-10 & 998,999          & 899,100           & 99,899            & 10.94           \\
            & Expert-random-30 & 998,999          & 699,300           & 299,699           & 8.44            \\
            & Expert-random-50 & 998,999          & 499,500           & 499,499           & 5.95            \\
            & Expert-random-70 & 998,999          & 299,700           & 699,299           & 3.46            \\ \bottomrule
\end{tabular}
\end{adjustbox}
\vspace{-0.3cm}
\end{table}

\paragraph{The contaminated D4RL adroit datasets.} The contaminated D4RL Adroit datasets can be constructed in a similar way, except that the non-expert trajectories are from cloned agents, i.e., imitation policies trained from the human-level demonstrations. Statistics of the contaminated D4RL Adroit datasets used in our evaluations are:
\vspace{-0.3cm}
\begin{table}[h]
\centering
\caption{Statistics of the contaminated D4RL Adroit Datasets (expert-cloned).}
\begin{adjustbox}{width={\textwidth},totalheight={\textheight},keepaspectratio}%
\renewcommand{\arraystretch}{1.1}
\begin{tabular}{cccccc}
\toprule
\textbf{Task} & \textbf{Setting} & \textbf{Total transition} & \textbf{Expert transition} & \textbf{Cloned transition} & \textbf{Averaged reward} \\ \hline
\multirow{4}*{Door}
         & Expert-cloned-10 & 995,000          & 895,500           & 99,500            & 13.08           \\
         & Expert-cloned-30 & 995,000          & 696,500           & 298,500           & 10.13           \\
         & Expert-cloned-50 & 995,000          & 497,500           & 497,500           & 7.16            \\
         & Expert-cloned-70 & 995,000          & 298,500           & 696,500           & 4.83            \\ \hline
\multirow{4}*{Hammer}   
         & Expert-cloned-10 & 995,000          & 895,500           & 99,500            & 55.31           \\
         & Expert-cloned-30 & 995,000          & 696,500           & 298,500           & 42.79           \\
         & Expert-cloned-50 & 995,000          & 497,500           & 497,500           & 30.06           \\
         & Expert-cloned-70 & 995,000          & 298,500           & 696,500           & 19.32           \\ \hline
\multirow{4}*{Pen}      
         & Expert-cloned-10 & 495,000          & 445,500           & 49,500            & 30.73           \\
         & Expert-cloned-30 & 495,000          & 346,500           & 148,500           & 25.96           \\
         & Expert-cloned-50 & 495,000          & 247,500           & 247,500           & 21.05           \\
         & Expert-cloned-70 & 495,000          & 148,500           & 346,500           & 20.61           \\ \hline
\multirow{4}*{Relocate} 
         & Expert-cloned-10 & 995,000          & 895,500           & 99,500            & 19.44           \\
         & Expert-cloned-30 & 995,000          & 696,500           & 298,500           & 15.10           \\
         & Expert-cloned-50 & 995,000          & 497,500           & 497,500           & 10.80           \\
         & Expert-cloned-70 & 995,000          & 298,500           & 696,500           & 8.27            \\ \bottomrule
\end{tabular}
\end{adjustbox}
\vspace{-0.3cm}
\end{table}

\subsection{Different state overlaps}
\label{app:different_state_overlaps}
In Figure \ref{fig:state_overlap}, We highlight two distinct situations involving different expert and non-expert state overlaps. When states visited by experts show great overlaps with non-expert states, the harmful closeness constraint toward non-expert decisions inhibits. For situation that expert states and non-expert states are well-distinguished, the failed closeness constraint happens as the learned policy is improved, showing different policy actions for dataset non-expert states. 

We here provided some visualizations of the distribution of expert and non-expert states in the expert-medium, expert-random, and expert-cloned settings.
\vspace{-0.3cm}
\begin{figure}[h]
    \centering
    \includegraphics[width=0.9\linewidth]{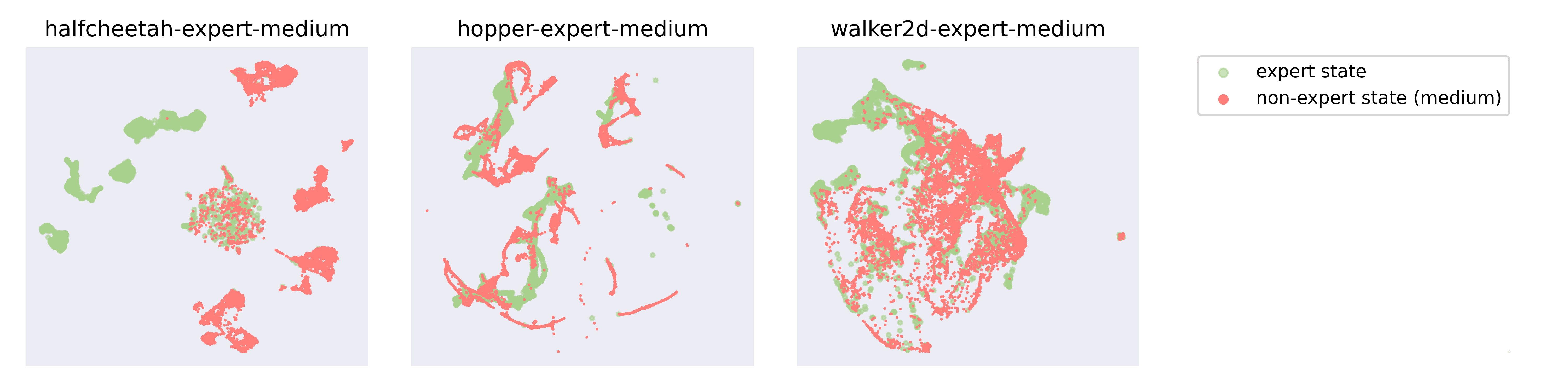}
    \includegraphics[width=0.9\linewidth]{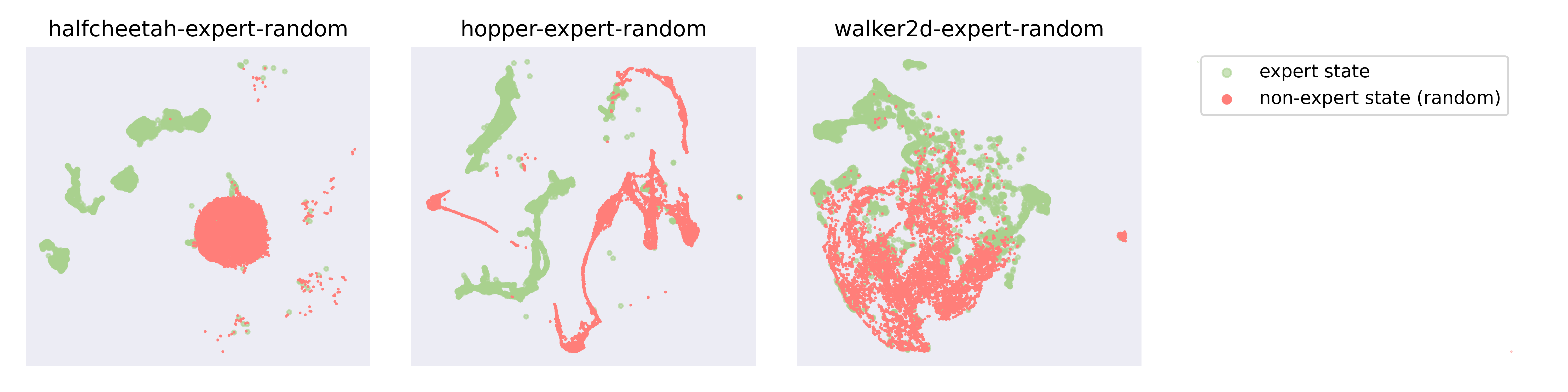}
    \includegraphics[width=0.9\linewidth]{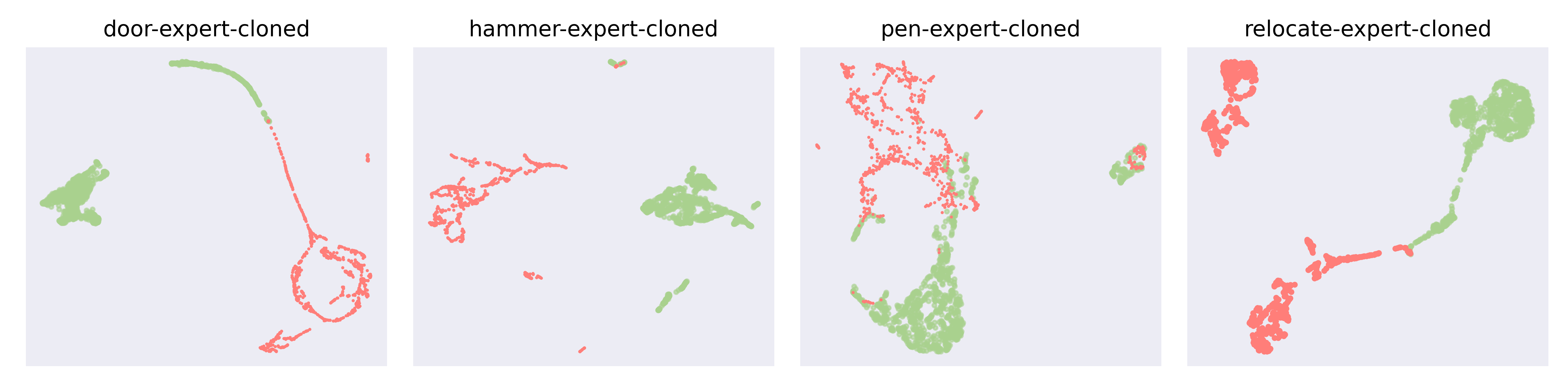}
    \vspace{-0.4cm}
    \caption{We use UMAP to reduce the dimensionality of states in different D4RL tasks. Expert states are visited by expert behavior policies, and non-expert states are from the medium, random or the cloned policies. We enlarge the dot size of expert states for clarity.}
    \vspace{-0.3cm}
    \label{fig:state_overlaps}
\end{figure}

\vspace{-0.25cm}
\paragraph{Great state overlaps.} In this situation, e.g., expert-medium datasets, states visited by expert-level behavior policies show great overlap with that of medium agents. Therefore, the closeness constraint towards non-expert actions may prevent the learned policy from moving closer to the expert decisions. Although offline RL with support-based policy constraints, e.g., BEAR, holds the promise to solve such issues, their exquisite metrics are often difficult to achieve. 
We alleviate the observed performance degradation by introducing a Q-weight for the policy constraint based method (+CR).

\vspace{-0.25cm}
\paragraph{Less state overlaps.} For datasets contaminated by low-level demonstrations, e.g., random and cloned level data, the expert and non-expert states show greatly different distributions. In this case, policy improvement inevitably changes the policy actions on non-expert states, increasing the probability of generating OOD decisions. This can be dangerous as OOD actions have been widely recognized as the source of exploding value function and the failed learning process. We suppress the OOD actions with the proposed GP technique. 

\vspace{-0.25cm}
\paragraph{The success of BC on adriot tasks.} For the simple mujoco tasks (controlling 3 or 6-DoF robotics), states visited by expert policies show great overlap with those visited by non-expert policies. With overlapped states, constraints toward non-expert actions affect the decision quality on expert states. In contrast, such impacts are eliminated with fewer overlaps under the complex Adroit tasks (24-DoF robotics). The records of non-expert state-action pairs less influence the decisions for expert states, thus leading to the success of BC agents on complex Adroit tasks.

\subsection{The harmful and the failed closeness constraint.}

\begin{figure}[h]
    \centering
    \includegraphics[width=0.3\linewidth]{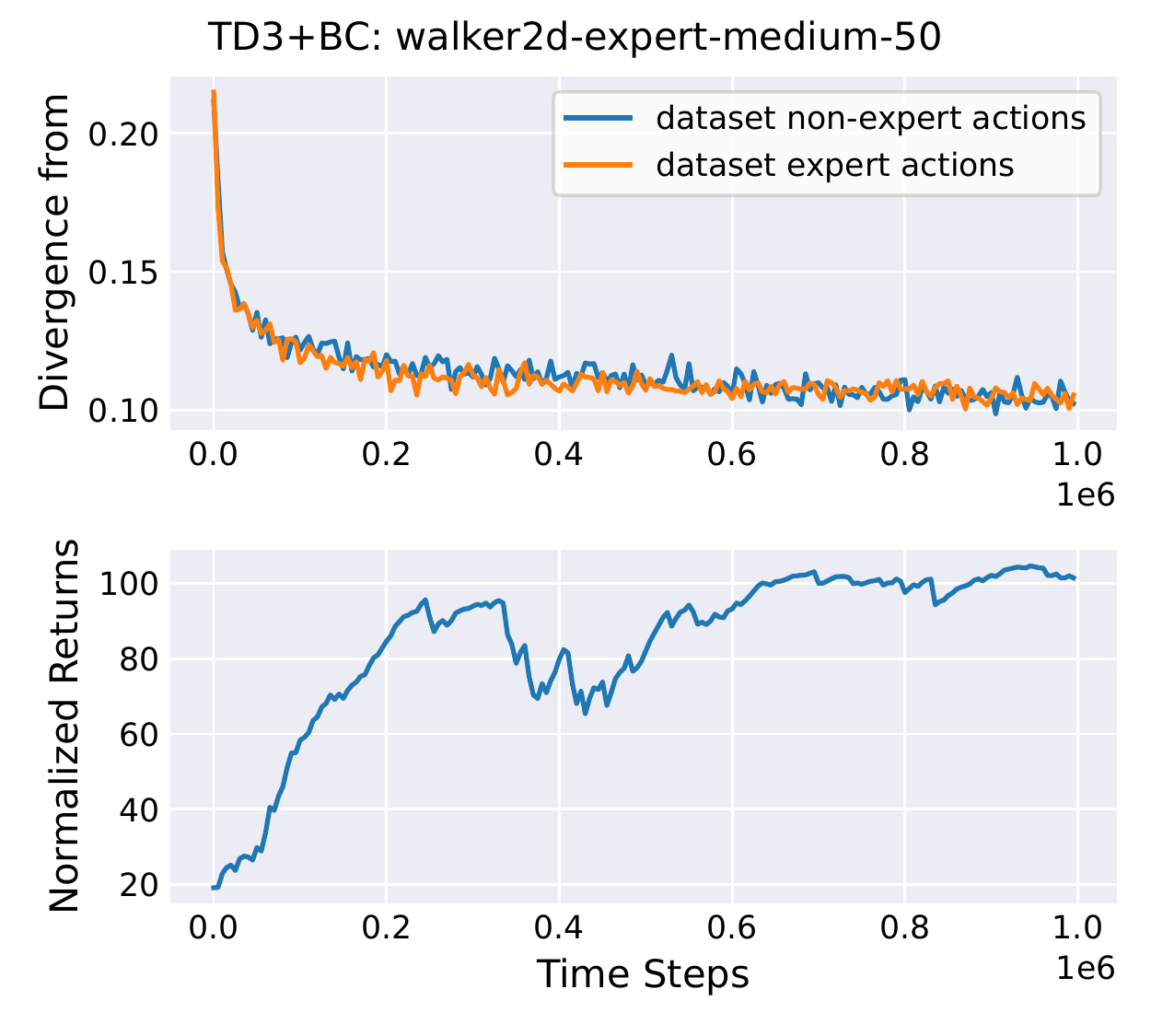}
    \includegraphics[width=0.3\linewidth]{pics/Figure_The_Failed_Closeness_Constraint.pdf}
    \includegraphics[width=0.3\linewidth]{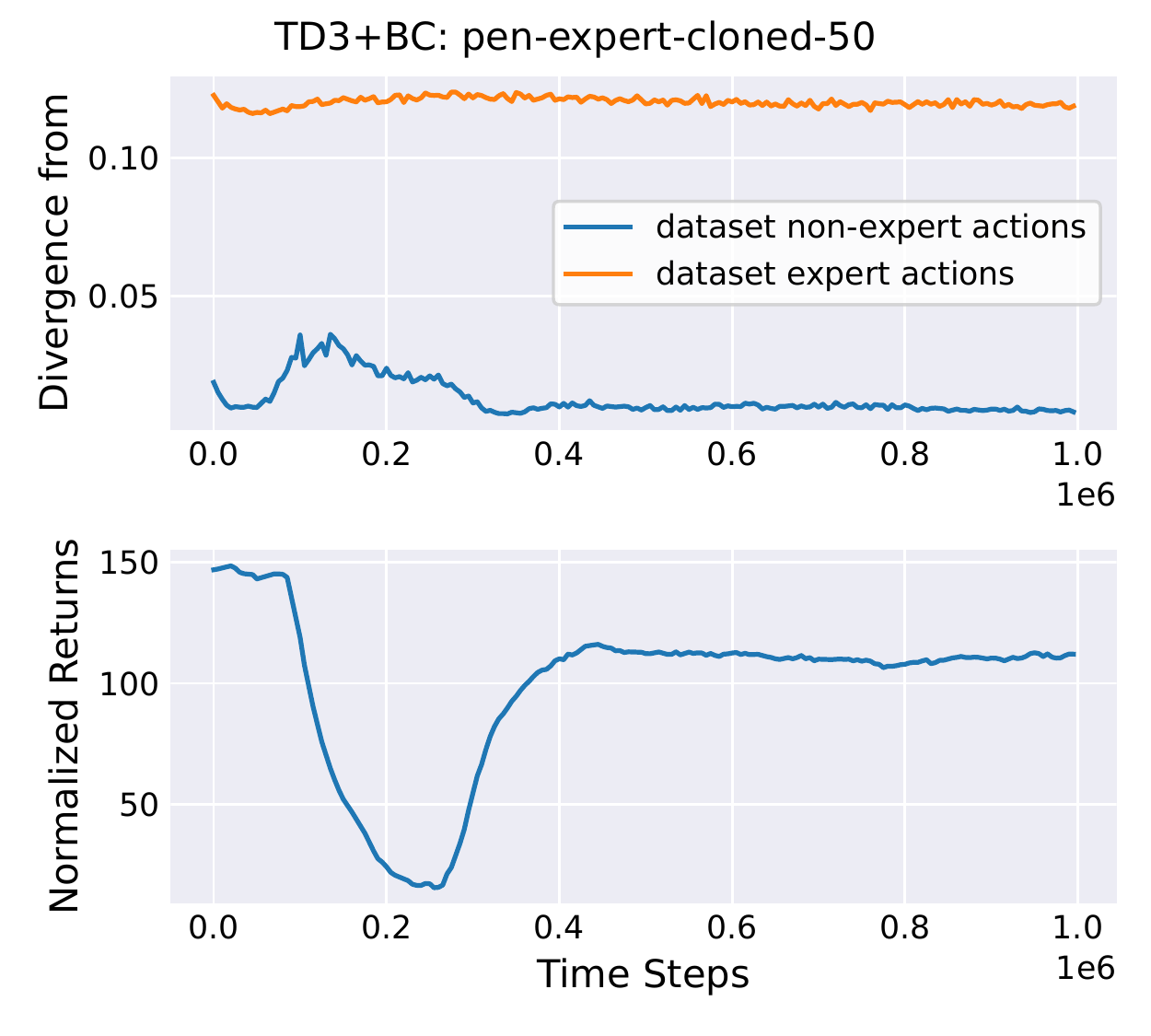}
    \vspace{-0.4cm}
    \caption{Visulation of the harmful closeness constraint (left) and the faild closeness constraint (middle and right). \textbf{Divergence:} the 75th percentile of the squared error between the decisions from the learned policy and the corresponding dataset actions.}
    \vspace{-0.3cm}
    \label{fig:visulation_harmful_and_failed}
\end{figure}

\paragraph{The harmful closeness constraint.} When expert and non-expert behavior policies share great state overlaps, two similar dataset states may correspond to two (or more) completely different actions. The closeness constraint towards non-expert one would inhibit the policy improvement in a supervised fashion. 

We visualize it in Figure~\ref{fig:visulation_harmful_and_failed}, left. The divergence between decisions from the learned policy and the expert behaviors becomes smaller as the policy improves. However, the distance to the non-expert dataset actions also becomes smaller. This contributes to the observed performance degradation. 
Although ideal support-based policy constraint methods hold the promise to handle this situation, empirically, their performance deteriorates. 

\vspace{-0.2cm}
\paragraph{The failed closeness constraint.} The main contribution of this paper is the finding that the policy improvement induces the failed closeness constraint on non-expert dataset states. That is, the policy improvement implicitly drives the learned policy to be different from the decisions recorded for non-expert states, inducing dangerous OOD actions. 

In the middle and the right-hand side of Figure~\ref{fig:visulation_harmful_and_failed}, we visualize the failed closeness constraint on non-expert dataset states, which happens after the policy achieves a good performance.

\subsection{Remarks}

\vspace{-0.2cm}
\paragraph{Are expert-random datasets too extreme for mimicking real-life scenarios?} The proposed contaminated dataset can be used to simulate the training behavior on a dataset containing two distinct behavior policies. In this context, what matters is not the non-expert behavior policies' quality but the states' overlap between the experts and non-experts. 

Such datasets do not necessarily have to be constructed by expert and random policies. For example, the catastrophic failures on expert-cloned datasets, see Table~\ref{tab:performance_on_adriot}, indicate that the learned Q-functions are destroyed by the sharp Q-function gradients, though the cloned behavior policies are far away from randoms.

\vspace{-0.2cm}
\paragraph{Difference with the D4RL replay datasets.}
This work focuses on training from contaminated datasets, including three different instances, expert-medium, expert-cloned, and expert-random. Another similar setting is the medium-replay or the full-replay dataset, which records all the interactions during the training. However, we think there is a significant difference between the two settings. 

Firstly, the contaminated dataset better fits offline reinforcement learning scenarios. Recall that the primary motivation of offline RL is to avoid the risky interactions for training policy from random initializations. Thus, it is unfeasible to collect logs like medium-replay or full-replay datasets in most situations. On the other hand, the contaminated dataset is used to simulate the training behavior on a dataset where two different behavioral policies exist. We believe a dataset with multiple behavior policies is a really common setting for real-life applications. 

Another difference is that the medium-replay and full-replay dataset have a wider distribution of state-action pairs and thus a lower probability of inducing OOD actions than the contaminated dataset considered in this paper. The proportions of expert (medium) trajectories in the replay datasets may also be smaller.

\section{Lipschitz property of the learned Q-function over action domain}
\label{app:lipschitzness}
In this section, we provide proof for Theorem~\ref{theo:lipschitzness}. In order to prove the desired Lipschitz continuity property of the learned Q-function, we need to give an upper bound of the magnitude of the Q-function gradient with respect to the input action, i.e., $\| \frac{\partial{Q(s_t, a_t)}}{\partial{a_t}}\|_F$ is bounded. 

For notational clarity, we use $\mathbb{E}_{s_{t+k} | s_t}[\cdot]$ to denote the expectation of the argument with respect to the conditional distribution of future state $s_{t+k}$ given that the agent starts from current $s_t$ and follows the policy $\pi(a_t|s_t)$, i.e., $\mathbb{E}_{s_t} \Big[ \big(\prod_{j=0}^k \pi(a_{t+j} | s_{t+j}) T(s_{t+1+j} | s_{t+j} , a_{t+j}) \big) [\cdot] \Big]$. Then we can rewrite the learned Q-function $Q^{\pi}(s_t, a_t)$ as $\sum_{k=0}^{\infty} \gamma^k \mathbb{E}_{s_{t+k} | s_t} [r(s_{t+k},a_{t+k})]$. Our proof starts from the upper bound of the Jacobian of the Q-function w.r.t one dimension of the action space. In such a case, we denote the $i-th$ dimension of the action space as $a_{t}^i$. We then drive to the case of multi-dimensional action space and complete our proof.

\begin{proposition}
\label{prop:one_dimension_lemma}
Suppose a policy $\pi$ on an MDP $M=\langle \mathcal{S}, \mathcal{A}, T, d_0, r, \gamma \rangle$ satisfies the following inequality for any given non-negative integer $t$:
\begin{equation}
    \Big\| \frac{\partial{\pi(a_{t+1} | s_{t+1}) }}{\partial{a_t}} \Big\|_F \leq L_{\pi, T},
\end{equation}

then it holds for any given non-negative integer $k$, and $t$:
\begin{equation}
    \Big| \nabla_{a_t^i} \mathbb{E}_{s_{t+k}|s_t} [r(s_{t+k}, a_{t+k})] \Big|
    \leq L_{\pi, T} \cdot \mathbb{E}_{s_{t+1}|s_t}  \Big| \nabla_{a_{t+1}^i} \mathbb{E}_{s_{t+k}|s_{t+1}} [r(s_{t+k}, a_{t+k})] \Big|.
\end{equation}
\end{proposition}

\begin{proof}
\label{proof:one_dimension_lemma}
\begin{align*}
    \Big| \nabla_{a_t^i} \mathbb{E}_{s_{t+k}|s_t} [r(s_{t+k}, a_{t+k})] \Big|
    &= \Big| \nabla_{a^i_{t+1}} \mathbb{E}_{s_{t+1}|s_t} \mathbb{E}_{s_{t+k}|s_{t+1}} [r(s_{t+k}, a_{t+k})] 
    \cdot \frac{\partial{a^i_{t+1}}}{\partial{a^i_t}}\Big| \\
    &\leq \Big| \nabla_{a^i_{t+1}} \mathbb{E}_{s_{t+1}|s_t} \mathbb{E}_{s_{t+k}|s_{t+1}} [r(s_{t+k}, a_{t+k})] \Big|
    \cdot \Big|\frac{\partial{a^i_{t+1}}}{\partial{a^i_t}}\Big| \\
    &=\Big|\frac{\partial{a^i_{t+1}}}{\partial{a^i_t}}\Big| \cdot 
    \mathbb{E}_{s_{t+1}|s_t} \Big| \nabla_{a^i_{t+1}} \mathbb{E}_{s_{t+k}|s_{t+1}} [r(s_{t+k}, a_{t+k})] \Big| \\
    &\leq L_{\pi, T} \cdot \mathbb{E}_{s_{t+1}|s_t}  \Big| \nabla_{a^i_{t+1}} \mathbb{E}_{s_{t+k}|s_{t+1}} [r(s_{t+k}, a_{t+k})] \Big| \\
\end{align*}
\end{proof}

The above proposition gives a derivation from a mild assumption, which is helpful for our next step proof. 

\begin{proposition}
\label{prop:one_dimension}
Suppose a policy $\pi$ on an MDP $M=\langle \mathcal{S}, \mathcal{A}, T, d_0, r, \gamma \rangle$ satisfies the following inequality for any given non-negative integer $t$:
\begin{align}
    \Big\| \frac{\partial{\pi(a_{t+1} | s_{t+1})}}{\partial{a_t}} \Big\|_F \leq L_{\pi,T} \\
    \Big\| \frac{\partial{r(s_t, a_t )}}{\partial{a_t}}  \Big \|_F \leq L_r, \quad
\end{align}

then it holds for any given non-negative integer $t$:
\begin{equation}
    \Big| \nabla_{a_t^i} \mathbb{E}_{s_{t+k}|s_t} [r(s_{t+k}, a_{t+k})] \Big| \leq L_{\pi,T}^k \cdot L_r.
\end{equation}
\end{proposition}

\begin{proof}
\label{proof:one_dimension}
\begin{align*}
    \Big| \nabla_{a_t^i} \mathbb{E}_{s_{t+k}|s_t} [r(s_{t+k}, a_{t+k})] \Big| 
    &\leq L_{\pi,T} \cdot \mathbb{E}_{s_{t+1}|s_t}  \Big| \nabla_{a_{t+1}^i} \mathbb{E}_{s_{t+k}|s_{t+1}} [r(s_{t+k}, a_{t+k})] \Big| \\
    &\leq L_{\pi,T} \cdot \mathbb{E}_{s_{t+1}|s_t} \cdots L_{\pi,T} \cdot \mathbb{E}_{s_{t+k}|s_{t+k-1}} \Big| \nabla_{a_{t+k}^i} \mathbb{E}_{s_{t+k}|s_{t+k}} [r(s_{t+k}, a_{t+k})] \Big| \\
    &= L_{\pi,T}^k \cdot \mathbb{E}_{s_{t+k}|s_t}  \Big| \nabla_{a_{t+k}^i} \mathbb{E}_{s_{t+k}|s_{t+k}} [r(s_{t+k}, a_{t+k})] \Big| \\
    &= L_{\pi,T}^k \cdot \mathbb{E}_{s_{t+k}|s_t} \Big| \nabla_{a_{t+k}^i} r(s_{t+k}, a_{t+k}) \Big| \\
    &\leq L_{\pi,T}^k \cdot \mathbb{E}_{s_{t+k}|s_t} \cdot L_r \\
    &= L_{\pi,T}^k \cdot L_r
\end{align*}
\end{proof}

Then we consider the case of multi-dimensional action space. An upper bound formulation of the learned Q-function gradient w.r.t. action can be derived by using Proposition \ref{prop:one_dimension_lemma} and Proposition \ref{prop:one_dimension}.

\textbf{Theorem \ref{theo:lipschitzness}.} \textit{ Suppose a policy $\pi(a_t | s_t)$ on an MDP $M=\langle \mathcal{S}, \mathcal{A}, T, d_0, r, \gamma \rangle$ satisfies the inequality $ \Big\| \frac{\partial{\pi(a_{t+1} | s_{t+1})}}{\partial{a_t}} \Big\|_F \leq L_{\pi,T} < 1$ and the reward function $r(s_t, a_t)$ satisfies $\Big\| \frac{\partial{r(s_t, a_t )}}{\partial{a_t}}  \Big \| \leq L_r$. If we denote the dimension of the action space as $N$, then the magnitude of the gradient of the learned Q-function w.r.t. action can be upperbounded as:}
\begin{equation}
\label{euqation:upperbound_in_app}
    \Big\| \nabla_{a_t} Q^{\pi}(s_t,a_t) \Big\|_F \leq \frac{\sqrt{N} L_r}{1- \gamma L_{\pi, T}}.
\end{equation}

\begin{proof}
\label{proof:multi_dimension}
\begin{align*}
    \Big\| \nabla_{a_t} Q^{\pi}(s_t,a_t) \Big\|_F^2 
    &= \sum_{i=0}^N \Big( \nabla_{a_t^i} Q^{\pi}(s_t, a_t)\Big)^2 \\
    &= \sum_{i=0}^N \Big( \sum_{k=0}^{\infty} \gamma^k \nabla_{a_t^i} \mathbb{E}_{s_{t+k} | s_t}[r(s_{t+k}, a_{t+k})] \Big)^2 \\
    &\leq \sum_{i=0}^N \Big( \sum_{k=0}^{\infty} \gamma^k \Big| \nabla_{a_t^i}  \mathbb{E}_{s_{t+k} | s_t}[r(s_{t+k}, a_{t+k})] \Big|\Big)^2 \\
    &= \sum_{i=0}^N \Big( \sum_{k=0}^{\infty} \gamma^k \cdot L_{\pi, T}^k \cdot L_r \Big)^2 \\
    &= N \big( L_r \sum_{k=0}^{\infty} (\gamma L_{\pi, T})^k \big)^2,
\end{align*}
finally, we have:
\begin{align*}
    \Big\| \nabla_{a_t} Q^{\pi}(s_t,a_t) \Big\|_F 
    &\leq \sqrt{N} L_r \sum_{k=0}^{\infty} (\gamma L_{\pi, T})^k \\
    &=  \frac{\sqrt{N} L_r}{1- \gamma L_{\pi, T}}
\end{align*}
\end{proof}

To better understand the proposed bound (\ref{euqation:upperbound_in_app}), we give some perspective on the constants in this formulation. Clearly, $L_r$ is the Lipschitz constant of the reward function w.r.t. the input action. Then we consider the meaning of $L_{\pi, T}$. $ \Big\| \frac{\partial{\pi(a_{t+1} | s_{t+1})}}{\partial{a_t}} \Big\|_F $ measures the change in the policy action $a_{t+1}$ at next state $s_{t+1}$ if we give an infinitesimal perturbation in the current policy action $a_t$. We denote its upper bound as $L_{\pi, T}$ as the Jacobian is related with the policy $\pi$ and the environment dynamics $T$:
\begin{align*}
    \frac{\partial{\pi(a_{t+1} | s_{t+1})}}{\partial{a_t}} 
    &= \frac{\partial}{\partial{a_t}} \pi\Big(a_{t+1} | T(s_{t+1} | s_t, a_t)\Big) \\
    &=  \frac{\partial{T(s_{t+1} | s_t, a_t)}}{\partial{a_t}}  \cdot \frac{\partial{\pi(a_{t+1} | s')}}{\partial{s'}} \Big|_{s' = T(s_{t+1} | s_t, a_t)} 
\end{align*}

Then we can derive the upper bound of the Jacobian as:
\begin{align*}
    \Big \|  \frac{\partial{\pi(a_{t+1} | s_{t+1})}}{\partial{a_t}}  \Big \|_F 
    & =  \Big \| \frac{\partial{T(s_{t+1} | s_t, a_t)}}{\partial{a_t}}  \cdot \frac{\partial{\pi(a_{t+1}| s')}}{\partial{s'}} \Big|_{s' = T(s_{t+1} | s_t, a_t)}  \Big \|_F \\
    &\leq  \Big \| \frac{\partial{\pi(a_{t+1} | s_{t+1})}}{\partial{s_{t+1}}} \Big \|_F \cdot \Big \| \frac{\partial{T(s_{t+1} | s_t, a_t)}}{\partial{a_t}} \Big \|_F
\end{align*}
The proposed constant $L_{\pi,T} $ is related with two Lipschitz constants, the first one for the policy $\pi$ w.r.t. the state space and another one for the environment dynamics $T$ w.r.t. the action space.

We refer the interested readers to \cite{memarian2021robust} for the proof of the upper bound for optimal Q-function gradients w.r.t. state space. For the Lipschitz continuity of the value function, see \cite{rachelson2010locality}.

\section{Experiment details}
\label{app:experiment_details}

We run our experiments on a single machine with 8 RTX3090 GPUs. All D4RL datasets use the v0 version. 

\subsection{Baselines}

\begin{table}[ht]
    \centering
    \begin{adjustbox}{width={0.8\textwidth},totalheight={\textheight},keepaspectratio}%
    \renewcommand{\arraystretch}{1.1}
    \begin{tabular}{cccccccc}
    \toprule
          & Walker2d & Hopper & Halfcheetah & Door & Hammer & Pen & Relocate \\
    \hline
         min\_q\_weight & 10 & 20& 20 & 20 & - & 50 & - \\
    \bottomrule
    \end{tabular}
    \end{adjustbox}
    \caption{Hyperparameter for CQL. We sweep it within the range of \{5, 10, 20, 50, 100\}.}
    \label{tab:hyperparameter_CQL}
\end{table}

\begin{table}[ht]
    \centering
    \begin{adjustbox}{width={0.8\textwidth},totalheight={\textheight},keepaspectratio}%
    \renewcommand{\arraystretch}{1.1}
    \begin{tabular}{cccccccc}
    \toprule
          & Walker2d & Hopper & Halfcheetah & Door & Hammer & Pen & Relocate \\
    \hline
         f\_reg & 1 & 1& 1 & 5 & 10 & 0.01 & 0.1 \\
    \bottomrule
    \end{tabular}
    \end{adjustbox}
    \caption{Hyperparameter for Fisher-BRC.}
    \label{tab:hyperparameter_Fisher}
\end{table}

\begin{itemize}
    \item{CQL.} We use a modular PyTorch implementation of CQL\footnote{Code and license: https://github.com/young-geng/cql}. We are very sorry that we cannot reproduce it on Adroit hammer and relocate tasks. To be more specific, for these omitted, the final D4RL normalized scores we got, acoss all swept paremeters, are about zero (random). We thus have to omit these irrational scores to prevent distress or offense to other readers and authors. Table \ref{tab:hyperparameter_CQL} shows the hyperparameters used in our experiments. 

    \item{BEAR-QL.} We use the recommended Github implementation~\footnote{Code and license: https://github.com/rail-berkeley/d4rl\_evaluations}. We follow the recommended settings for mujoco tasks, and for four Adroit tasks, we use the Gaussian kernel.
    
    \item{UWAC.} We use the official implementation~\footnote{Code and license: https://github.com/apple/ml-uwac}, with default hyperparameters.
    
    \item{IQL.} We use the authors' implementaion in JAX~\footnote{Code and license: https://github.com/ikostrikov/implicit\_q\_learning}, which is really really fast. 
    
    \item{Fisher-BRC.} We use the author's implementation~\footnote{Code and license: https://github.com/google-research/google-research/tree/master/fisher\_brc}.  We sweep the best hyperparameters for D4RL Adroit expert tasks and follow the suggested settings for D4RL mujoco tasks. 
\end{itemize}

\subsection{The proposed method}

\paragraph{Implementation.} 

We recommend interested readers to reproduce results of TD3BC++ on the top of TD3+BC~\footnote{Code and license: https://github.com/sfujim/TD3\_BC}, which is really a minimalist approach to offline RL. The proposed plugin involves two algorithmic modifications:

\lstset{style=mystyle}
\begin{lstlisting}[
language=Python, 
caption=The proposed two small changes on the top of TD3+BC. 
]
    # Compute critic loss
    critic_loss = F.mse_loss(current_Q1, target_Q) + F.mse_loss(current_Q2, target_Q)
+   if self.total_it % N == 0: # We empirically set N to 5.
+   	_state_rep = state.clone().detach().repeat(16, 1).requires_grad_(True)
+   	_random_action = torch.rand(
+   	    size=self.actor(_state_rep).size(), 
+   	    requires_grad=True
+   	    ) * 2 - 1.0
+   	_random_action= _random_action.to(device)
+   	_current_Q1, _current_Q2 = self.critic(_state_rep, _random_action)
+   	grad_q1_wrt_random_action = torch.autograd.grad(
+   		outputs=_current_Q1.sum(),
+   		inputs =_random_action,
+   		create_graph=True
+   		)[0].norm(p=2, dim=-1)
+       grad_q2_wrt_random_action = torch.autograd.grad(
+   		outputs=_current_Q2.sum(),
+   		inputs =_random_action,
+   		create_graph=True
+   		)[0].norm(p=2, dim=-1)
+   	grad_q_wrt_random_action = F.relu(grad_q1_wrt_random_action - self.k) **2 +\
+           F.relu(grad_q2_wrt_random_action - self.k) **2
+   	critic_loss = critic_loss +  grad_q_wrt_random_action.mean() * self.lambda_GP
...
    # Compute actor loss
-   # actor_loss = -lmbda * Q. mean() + F. mse_loss(pi, action)
+   current_Q = ((current_Q1 + current_Q2) * 0.5).squeeze().detach()
+   actor_loss = -lmbda * Q.mean() + \
+       (F.mse_loss(pi, action, reduction='none').mean(axis=-1) * current_Q).mean()
\end{lstlisting}

\paragraph{Hyperparameters used for experiments.} Our modification involves a weight factor $\lambda_{GP}$ for gradient penalty loss $\mathcal{L}_{GP}$. As for the backbone algorithm, TD3+BC, we find $\alpha$, a factor to control the strength of BC term in Equation \ref{eq:td3_bc}, affects performance the most.  \cite{fujimoto2021minimalist} use $\alpha=2.5$ for their experiments on D4RL mujoco gym tasks. However, we find it does not work for Adroit tasks. We sweep it within the range of \{0.05, 0.1, 0.2, 0.5, 1, 2, 2.5, 3, 4\} and select the maximum possible value that works. Note that, TD3+BC with a low value of $\alpha$ may degenerate to imitation (BC term will dominate the learning) rather than RL. 
We report the settings used for our experiments:
\begin{table}[!hp]
    \centering
    \begin{adjustbox}{width={0.8\textwidth},totalheight={\textheight},keepaspectratio}%
    \renewcommand{\arraystretch}{1.1}
    \begin{tabular}{ccccccccc}
    \toprule
        & & Walker2d & Hopper & Halfcheetah & Door & Hammer & Pen & Relocate \\
    \hline
        TD3+BC & $\alpha$ & 2.5 & 2.5 & 2.5 & 0.5 & 0.2 & 0.5 & 0.02 \\
    \hline
    \multirow{2}{*}{TD3BC++} & $\alpha$ & 2.5 & 2.5 & 2.5 & 0.5 & 0.2 & 0.5 & 0.02 \\
                              & $\lambda_{GP}$ & 1 & 1 & 1 & 1 & 1 & 1 & 0.1 \\

    \hline
    \multirow{1}{*}{BEAR++} & $\lambda_{GP}$ & 1 & 1 & 1 & 1 & 1 & 1 & 0.1 \\
    \bottomrule
    \end{tabular}
    \end{adjustbox}
    \caption{Hyperparameters for TD3+BC, TD3BC++, and BEAR++.}
    \label{tab:hyperparameter_++}
\end{table}

\paragraph{Hyperparameter study.} We fix the BC term $\alpha=2.5$ and vary the gradient penalty term $\lambda_{GP}$ in TD3BC++, sweeping on four different settings. Results are shown in figure \ref{fig:hyperparameter_study_appendix}. For the EM setting, a small GP term (0.01, 0.02, or 0.05) can have a stabilizing effect on training while an overlarge one would inhibit the learned Q-function. As for the difficult ER3, ER5, and ER7 settings, we recommend practitioners choose a medium value (1, 2, 5) to stabilize learning while avoiding making the Q-function too flat.

\begin{figure}
    \centering
    \includegraphics[width=1.0\linewidth]{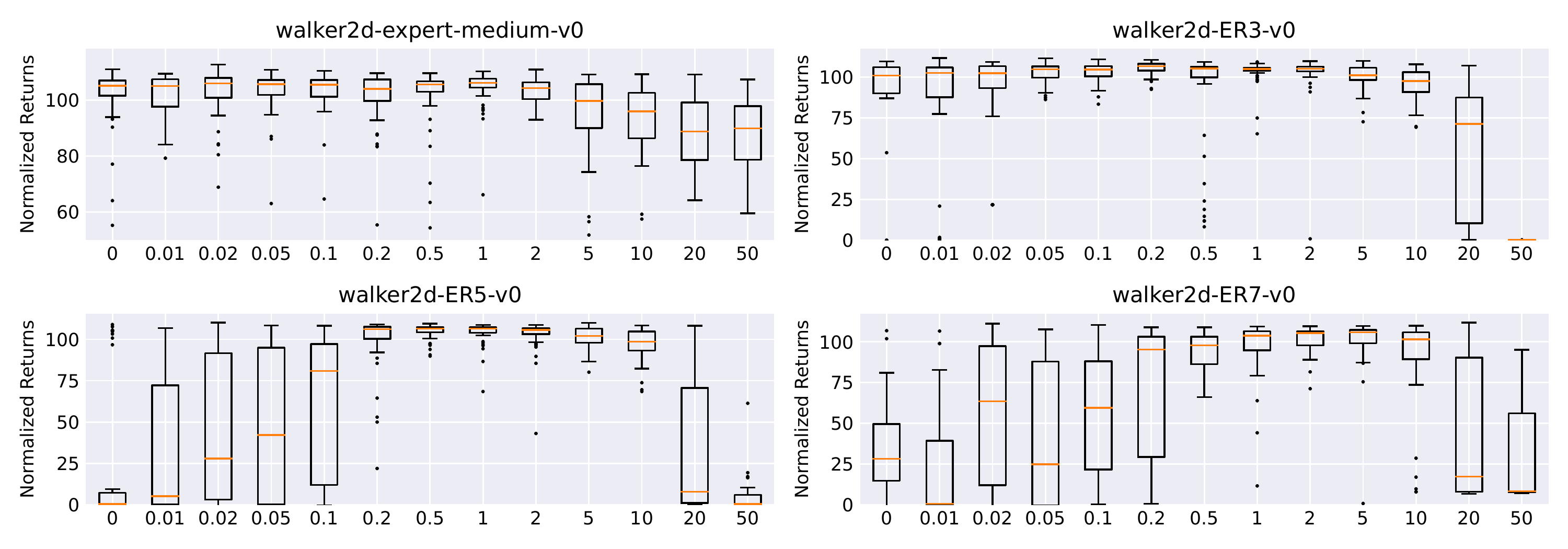}
    \caption{Hyperparameter study. Box plot. We run TD3BC++ with different $\lambda_{GP}$.}
    \label{fig:hyperparameter_study_appendix}
\end{figure}

\subsection{Evaluation on D4RL mujoco gym tasks.}
\label{app:evaluation_on_mujoco}
This work focuses on addressing the performance degradation and the catastrophic failure issues for policy constraint based offline RL algorithms. Therefore, we are more concerned with the performance on contaminated datasets with non-expert trajectories. In order to show the potential influence of the proposed methods, we report the results of BEAR++ and TD3BC++ on classic D4RL mujoco gym tasks, with hyperparameters in Table \ref{tab:hyperparameter_++}.

\vspace{-0.3cm}
\begin{table*}[!ht]
\centering
\caption{Evaluation on the original D4RL mujoco gym tasks.}
\label{tab:performance_on_original_mujoco_tasks}
\begin{adjustbox}{width={0.95\textwidth},totalheight={\textheight},keepaspectratio}%
\renewcommand{\arraystretch}{1.1}
\begin{tabular}{cc d{3.3}d{3.3}d{3.3}d{3.3}d{3.3}d{3.3}d{3.3}d{3.3}}
\toprule
\textbf{Task}   & \textbf{Setting}  & \multicolumn{1}{c}{\textbf{BC}}  & \multicolumn{1}{c}{\textbf{CQL}} & \multicolumn{1}{c}{\textbf{Fisher-BRC}} &  \multicolumn{1}{c}{\textbf{AWAC}} & \multicolumn{1}{c}{\textbf{BEAR}} & \multicolumn{1}{c}{\textbf{TD3+BC}} & \multicolumn{1}{c}{\textbf{BEAR ++}} & \multicolumn{1}{c}{\textbf{TD3BC ++}} \\
\hline
\parbox[t]{1mm}{\multirow{4}{*}{\rotatebox[origin=c]{90}{Halfcheetah}}}
            & Expert        & 105.20 & 82.40  & 108.40     & 78.50 & 103.77 & 105.70 & 104.53 & 105.87   \\
            & Medium-expert & 67.60  & 27.10  & 93.30      & 36.80 & 49.25  & 97.90  & 91.01  & 105.26   \\
            & Medium        & 36.60  & 37.20  & 41.30      & 37.40 & 37.09  & 42.80  & 36.85  & 40.78    \\
            & Random        & 2.00   & 21.70  & 33.30      & 2.20  & 2.26   & 10.20  & 2.25   & 6.98     \\
             \hline
\parbox[t]{1mm}{\multirow{4}{*}{\rotatebox[origin=c]{90}{Hopper}}}
            & Expert        & 111.50 & 111.20 & 112.30     & 85.20 & 61.50  & 112.20 & 111.36 & 112.23   \\
            & Medium-expert & 89.60  & 111.40 & 112.40     & 80.90 & 85.12  & 112.20 & 110.28 & 111.57   \\
            & Medium        & 30.00  & 44.20  & 99.40      & 72.00 & 37.89  & 99.50  & 39.84  & 30.78    \\
            & Random        & 9.50   & 10.70  & 11.30      & 9.60  & 10.22  & 11.00  & 10.04  & 10.58    \\
             \hline
\parbox[t]{1mm}{\multirow{4}{*}{\rotatebox[origin=c]{90}{Walker2d}}}        
            & Expert        & 56.00  & 103.80 & 103.00     & 57.00 & 75.13  & 105.70 & 97.20  & 104.68   \\
            & Medium-expert & 12.00  & 68.10  & 105.20     & 42.70 & 56.08  & 101.10 & 74.13  & 104.46   \\
            & Medium        & 11.40  & 57.50  & 78.80      & 30.10 & 57.87  & 79.70  & 62.46  & 75.79    \\
            & Random        & 1.20   & 2.70   & 1.50       & 5.10  & 3.27   & 1.40   & 19.90  & 5.26     \\
            \bottomrule
\end{tabular}
\end{adjustbox}
\end{table*}

We note that TD3BC++ performs much lower than TD3+BC on the Hopper medium task and returns to the original performance after reducing the strength of the GP term ($\lambda_{GP}=0.02$, Score=100.13, 5 seeds). This indicates that one needs to select a suitable $\lambda_{GP}$ value, based on the quality of the dataset and the difficulty of the task.

\section{Broader impact}
Policy constraint based offline RL is a crucial approach to data-driven decision-making machines. As it enjoys many advantages, such as easy implementation, small training costs, and no need for extensive domain knowledge, one can apply it to various scenarios. Therefore, we believe that this work will inevitably inherit the social impact of application contexts. 

It would be surprising to see that the proposed plugins alleviate the observed performance degradation and catastrophic failure issues for policy constrained offline RL. With them, one can make greater use of static demonstrations to obtain stronger agents. To this degree, we believe our social impact lies in expanding the applicability of policy constraint based offline reinforcement learning methods.
\end{document}